\documentclass{article}

\usepackage[preprint]{neurips_2026}

\usepackage[utf8]{inputenc} 
\usepackage[T1]{fontenc}    
\usepackage{hyperref}       
\usepackage{url}            
\usepackage{booktabs}       
\usepackage{amsfonts}       
\usepackage{nicefrac}       
\usepackage{microtype}      
\usepackage{xcolor}         

\usepackage{amsmath, amssymb, amsthm}
\newtheorem{corollary}{Corollary}
\usepackage{algorithm}
\usepackage{algorithmic}
\usepackage{graphicx}
\usepackage{subcaption}

\title{Temper and Tilt Lead to SLOP: Reward Hacking Mitigation with Inference-Time Alignment}

\author{
Ye Wang \& Jing Liu \& Toshiaki Koike-Akino \\
Mitsubishi Electric Research Laboratories (MERL) \\
201 Broadway, Cambridge, MA 02139, USA \\
\texttt{\{yewang,jiliu,koike\}@merl.com}
}

\begin{document}

\maketitle

\begin{abstract}
Inference-time alignment techniques offer a lightweight alternative or complement to costly reinforcement learning, while enabling continual adaptation as alignment objectives and reward targets evolve. Existing theoretical analyses justify these methods as approximations to sampling from distributions optimally tilted toward a given reward model. We extend these techniques by introducing reference-model temperature adjustment, which leads to further generalization of inference-time alignment to ensembles of generative reward models combined as a sharpened logarithmic opinion pool (SLOP). To mitigate reward hacking, we propose an algorithm for calibrating SLOP weight parameters and experimentally demonstrate that it improves robustness while preserving alignment performance.
\end{abstract}

\section{Introduction}

Alignment through post-training with reinforcement learning (RL) is a widely used and crucial technique for realizing effective foundation models~\citep{christiano2017deep, ziegler2019fine, ouyang2022training, korbak2022reinforcement, bai2022training, rafailov2023direct, shao2024deepseekmath}.
Inference-time alignment methods, such as Best-of-N (BoN)~\citep{stiennon2020learning, nakano2021webgpt}, provide a lightweight alternative or complement to computationally expensive RL methods, while achieving competitive performance~\citep{gao2023scaling, dubois2023alpacafarm}.
Further, inference-time methods enable flexible, continual adaptation to evolving alignment objectives.

Common across many RL approaches is the optimization of the policy $\pi$ to maximize expected reward, subject to Kullback--Leibler (KL) regularization with respect to a reference policy $p$, i.e., $\max_\pi \mathbb{E}_\pi[r] - \lambda \mathrm{KL}(\pi \| p)$.
\citet{korbak2022rl} observed that, in principle, the optimal policy is a tilted distribution in the form $\pi^* \propto p \exp(r / \lambda)$, which is also known from information projection problems in classical information theory literature~\citep{kullback1966note, csiszar1975divergence}.
This perspective has been a key to a line of research~\citep{yang2024asymptotics, mroueh2025information, verdun2025soft, khalaf2025inference} that has considered BoN and other related inference-time methods, while providing theoretical analysis of how they approximate sampling the optimal tilted distribution policy.

Our work extends upon these inference-time alignment methods, by introducing reference model temperature adjustment, which is closely related to a proposal of~\citet{jinnai2025regularized}, and further generalizing this framework to ensembles of generative reward models as a form of logarithmic opinion pool (LOP)~\citep{heskes1997selecting}.
This perspective, while bridging RL objectives, sampling-based methods, and ensemble modeling, also reveals the impact of correlation and the benefits of diversity, when combining models.
We experimentally demonstrate the effectiveness of calibrating the ensemble weights to mitigate reward hacking, while preserving alignment performance, for visual question answering and math reasoning tasks.

\section{Formulation, background, and methods}

\subsection{Tempered reward maximization}

We consider a generalization of the KL-regularized reward maximization problem, given by
\begin{equation}
\label{eq:tempered_reward}
\pi_{\alpha, \beta}^*(y \mid x) =
\arg \max_\pi \mathbb{E}_{Y \sim \pi(\cdot \mid x)}
  \big[ \beta r(x,Y) + (\alpha - 1) \log p(Y \mid x) \big]
  - \mathrm{KL}\big( \pi(y \mid x) \| p(y \mid x) \big),
\end{equation}
where $r: \mathcal{X} \times \mathcal Y \to \mathbb{R}$ is a given (proxy) reward model, $p(y \mid x)$ is a given reference model, and parameters $\alpha, \beta \in \mathbb{R}$ control the regularization.
Typically, $\beta > 0$ is used, but we allow for $\beta \leq 0$ to consider situations where the proxy reward model may be severely misaligned with the ideal reward.

The standard KL-regularized reward maximization problem is the special case of~\eqref{eq:tempered_reward} with $\alpha = 1$, where the reference model log-likelihood term disappears from the expectation, and the optimal solution is shown by~\citet{korbak2022rl} to be
\begin{equation}
\label{eq:tilted}
\pi_{1, \beta}^*(y \mid x) = 
\frac{p(y \mid x) \exp(\beta r(x, y))}{C_{\beta, x}},
\end{equation}
where $C_{\beta, x} := \int p(y \mid x) \exp(\beta r(x,y)) \, dy$ is a normalizing constant.\footnote{For discrete $\mathcal{Y}$, the normalizing constants throughout the paper are instead defined via summation.}
The \emph{tilted distribution} solution of~\eqref{eq:tilted} is also known in classical information theory literature~\citep{kullback1966note, csiszar1975divergence}.
Note that as $\beta \to \infty$, the optimal policy is plain reward maximization over $y$ in the support of the reference model, while for $\beta = 0$, the optimum is just the reference model, i.e., $\pi_{1, 0}^* = p$, due to regularization.
Sweeping the value of $\beta$, these tilted distributions achieve the optimal reward maximization versus KL divergence (from the reference model) trade-off, which is the target for many RL training methods.
Characterizing the general problem of~\eqref{eq:tempered_reward} follows as a corollary.

\begin{corollary}
\label{co:temper_tilt}
The problem given above in~\eqref{eq:tempered_reward} is equivalently written and solved as
\begin{align}
\label{eq:alt_tempered_reward}
\pi_{\alpha, \beta}^*(y \mid x) &=
\arg \max_\pi \mathbb{E}_{Y \sim \pi(y \mid x)}
  \big[ \beta r(x,Y) \big]
  - \mathrm{KL}\big( \pi(y \mid x) \| p(y \mid x)^\alpha / C_{\alpha, x} \big) \\
&= \frac{p(y \mid x)^\alpha \exp(\beta r(x, y))}{C_{\alpha, \beta, x}}, \label{eq:temper_tilted}
\end{align}
where $C_{\alpha, x} := \int p(y \mid x)^\alpha \, dy$ and $C_{\alpha, \beta, x} := \int p(y \mid x)^\alpha \exp(\beta r(x, y)) \, dy$ are normalizing constants.
\end{corollary}

Corollary~\ref{co:temper_tilt} states that augmenting the reward with an additional reference model log-likelihood term in~\eqref{eq:tempered_reward} is equivalent to adjusting the temperature of the reference model in~\eqref{eq:alt_tempered_reward}.
These equivalent problems are both solved by~\eqref{eq:temper_tilted}, which we call the \emph{tempered and tilted} distribution, as $\alpha$ is essentially an inverse temperature applied to the reference model, while $\beta$ controls the exponential tilting towards the reward model.
Note that as $\alpha \to \infty$, with a fixed, finite $\beta$, the solution approaches the maximum likelihood output of the reference model, whereas $\alpha = 0$ makes the reference model uniform and removes its impact, and values of $\alpha < 0$ inverts the reference model, amplifying less likely outputs.

\subsection{Approximate inference-time alignment techniques}

While the tilted (and tempered) distributions of~\eqref{eq:tilted} and~\eqref{eq:temper_tilted}, in principle, express closed-form solutions for the regularized reward maximization problems given by~\eqref{eq:tempered_reward} and~\eqref{eq:alt_tempered_reward}, handling these formulations directly is often computationally intractable.
Even with finite sets, where the integration for determining the normalizing constants are replaced with summation, these calculations become infeasible due to the enormous size of the output space (e.g., all token sequences up to a max length).
Hence, this subsection reviews some practical inference-time methods that approximate sampling the tilted distribution of~\eqref{eq:tilted}, i.e., for the special case of $\alpha = 1$, while requiring only the ability to draw independent samples of the reference model and to evaluate the reward model on those samples.

Best-of-N (BoN) is a widely adopted heuristic (see~\citep{stiennon2020learning} and~\citep{nakano2021webgpt} for early influential examples) to augment or replace RL training.
It simply consists of generating $n$ independent samples, $y_1, \ldots, y_n$, from the reference model $p(y \mid x)$, evaluating each with the reward model, and outputting the sample that maximizes the reward, i.e., $\arg \max_{y_i} r(x, y_i)$.
Under certain simplifying assumptions, \citet{yang2024asymptotics} established that BoN asymptotically approaches the optimal tilted distribution.
However, in more general finite regimes, which are characterized by \citet{mroueh2025information}, BoN does not necessarily approximate the tilted distribution.
Best-of-Poisson (BoP)~\citep{khalaf2025inference} is an extension of BoN, where, instead of a fixed $n$, the number of samples is drawn from a Poisson distribution, which results in a close approximation of the tilted distribution, under the assumption of uniformly distributed rewards.

Soft Best-of-N (SBoN)~\citep{verdun2025soft} similarly generates $n$ independent samples and evaluates the reward, $r_i = r(x, y_i)$, for each sample.
However, instead of selecting the maximizer, SBoN randomly selects sample $y_i$ with probability $\exp(\beta r_i) / \sum_j \exp(\beta r_j)$, i.e., the distribution over the $n$ samples formed by the softmax of their rewards, with inverse temperature $\beta$.
\citet{verdun2025soft} characterize the accuracy of SBoN, which includes establishing that $\textrm{KL}(\pi_{1, \beta}^* \| \tilde{\pi}_{\beta, n}) = O(1/n)$,
where $\tilde{\pi}_{\beta, n}$ denotes the distribution sampled by SBoN.
Reward Augmented Decoding (RAD)~\citep{deng2023reward} employs a similar concept, but, instead of selecting among entire response candidates, RAD utilizes the reward model to augment the sampling of each token, within the iterative process of auto-regressive generation.

\citet{jinnai2025regularized} propose methods to augment the reward model for regularizing BoN to mitigate reward-hacking.
Their main proposal, inspired by minimum Bayes risk decoding, involves a proximity regularizer that aims to minimize the Wasserstein distance with respect to the reference policy.
As an ablation, \citet{jinnai2025regularized} also consider what they call KL-regularized BoN, which essentially augments the reward with the reference model log-likelihood, and is analogous to the $(\alpha - 1) \log p(y \mid x)$ term in the tempered reward maximization of~\eqref{eq:tempered_reward}.

\subsection{SLOP: sharpened logarithmic opinion pooling}

The proxy reward model $r$ implicitly defines the \emph{alignment distribution}, given by
\begin{align}
q(y \mid x) := \exp(r(x, y)) / R_x,
\end{align}
where $R_x := \int \exp(r(x, y)) \, dy$ is a normalizing constant.
The tempered and tilted distribution, given in~\eqref{eq:temper_tilted}, can be rewritten as
\begin{align}
\pi_{\alpha, \beta}^*(y \mid x)
&= \frac{p(y \mid x)^\alpha q(y \mid x)^\beta}{\int p(z \mid x)^\alpha q(z \mid x)^\beta \, dz} \\
&=: \mathrm{softmax}\big( \alpha \log p(y \mid x) + \beta \log q(y \mid x) \big),
\end{align}
where $\mathrm{softmax}(\cdot)$ denotes exponentiation and normalization, such that we have valid distributions over $y \in \mathcal{Y}$, for each $x \in \mathcal{X}$.
This form makes it clear that the tempered and tilted distribution is essentially a \emph{logarithmic opinion pool} (LOP)~\citep{genest1986combining, heskes1997selecting}, which is also related to the concept of \emph{product of experts}~\citep{hinton2002training}.
However, unlike literature that typically considers LOP with non-negative weights that sum to one, we instead employ arbitrary $\alpha, \beta \in \mathbb{R}$, which allows for concentrating probability mass on the most confident outputs (e.g., converging towards proxy reward maximization as $\beta \to \infty$) or even inverting the contribution of anti-aligned models (e.g., if the negated proxy reward is closer to the gold reward, then $\beta < 0$ may be more effective).
We emphasize this aspect by calling this form of model as a \emph{Sharpened LOP} (SLOP).\footnote{Admittedly, also for the sake of utilizing a memorable and ironic acronym.}

Inspired by this logarithmic pooling perspective, we can also generalize to a pool of $m$ experts (i.e., generative and/or proxy reward models), denoted by $s_1, \ldots, s_m$, representing the score contributed by each expert, where $s_i(x, y) = \log p_i(y \mid x)$ for a generative model and $s_i(x, y) = r_i(x, y) \equiv \log q_i(y \mid x)$ for a reward model.\footnote{Due to later normalization via softmax, the normalizing constant of the alignment distribution may be omitted.}
These experts are combined, with weights $\omega := (\omega_1, \ldots, \omega_m) \in \mathbb{R}^m$, to form the SLOP given by
\begin{align}
\label{eq:SLOP}
\pi_\omega^*(y \mid x) = \mathrm{softmax}\left( \sum_{i=1}^m \omega_i s_i(x, y) \right) \propto \prod_{i=1}^m p_i(y \mid x)^{\omega_i},
\end{align}
where, for notational convenience, each $p_i$ denotes either a generative model or the corresponding alignment distribution $q_i$ of a reward model.

Directly sampling the SLOP distribution from~\eqref{eq:SLOP} is often intractable, when the output space $\mathcal{Y}$ is very large.
However,~\eqref{eq:SLOP} can also be interpreted as the solution of a regularized reward maximization problem, where the proxy reward is a weighted ensemble of expert scores, which allows us to approximately sample from the SLOP via the following extension of SBoN~\citep{verdun2025soft}.
By convention, we set the generative reference model as the first expert $p_1$.
For a given prompt $x$, we sample $n$ candidates $y_1, \ldots, y_n \stackrel{\text{iid}}{\sim} p_1(\cdot \mid x)$, evaluate candidate rewards as function of expert scores, as given by,
\begin{align}
\label{eq:SLOP_pseudo_reward}
r'(x, y_j) := (\omega_1 - 1) \log p_1(y_j \mid x) + \sum_{i=2}^m \omega_i s_i(x, y_j),
\end{align}
and select one candidate according to
\begin{align}
\label{eq:SLOP_SBoN}
\widehat{\pi}_{\omega,n}(y_j \mid x)
= \mathrm{softmax}(r'(x, y_j))
\propto
\exp \left(
\sum_{i=1}^m \omega_i s_i(x,y_j)
-
s_1(x,y_j)
\right),
\end{align}
where $\mathrm{softmax}(\cdot)$ denotes exponentiation and normalization over the set of $n$ candidates.
Note that the first expert weight is set to $(\omega_1 - 1)$, which accounts for the candidates being already sampled from the first expert.

\begin{corollary}
\label{co:SLOP_SBON}
Let $\widehat{\pi}_{\omega,n}$ denote the distribution realized via the above extension of SBoN, by selecting from $n$ sampled candidates $y_1, \ldots, y_n \stackrel{\text{iid}}{\sim} p_1(\cdot \mid x)$, according to~\eqref{eq:SLOP_SBoN}.
This approximates sampling the SLOP distribution $\pi_{\omega}^*$ given in~\eqref{eq:SLOP}, with
$\textrm{KL}(\pi_{\omega}^* \| \widehat{\pi}_{\omega,n}) = O(1/n)$.
\end{corollary}

\begin{proof}
This corollary follows immediately from Theorem 1 of~\citep{verdun2025soft}, which shows that $\textrm{KL}(\pi \| \widehat{\pi}_{\omega,n}) = O(1/n)$, for the distribution $\pi \propto p_1 \exp(r')$, since $\widehat{\pi}_{\omega,n}$ is produced by essentially applying SBoN with the reference model $p_1$ and the reward model $r'$ given by~\eqref{eq:SLOP_pseudo_reward}.
Expanding $r'$, we have $p_1 \exp(r') = \prod_{i=1}^m \exp(\omega_i s_i)$, and thus $\pi = \pi_{\omega}^*$.
\end{proof}

In the special case of one reference model and one reward model, this approximates tempered-and-tilted sampling, with candidate likelihoods proportional to
$\exp\big( (\alpha-1)\log p(y_j \mid x) + \beta r(x,y_j) \big)$, where $(\alpha, \beta)$ correspond to $(\omega_1, \omega_2)$.

\subsection{Reward hacking mitigation via calibrated model pooling}

\begin{algorithm}[t]
\caption{Calibrating SLOP Weights}
\label{alg:calibrate_slop}
\begin{algorithmic}
\REQUIRE reference model $p_1$; expert models
$s_1,\ldots,s_m$, where $s_1(x,y)=\log p_1(y\mid x)$;
calibration prompts $x_1,\ldots,x_k$; candidates per prompt $n$;
gold reward $g$; steps $T$; learning rate $\eta$; weight decay parameter $\lambda$;

\FOR{$i=1$ to $k$}
    \FOR{$j=1$ to $n$}
        \STATE Sample candidates $y_{ij} \sim p_1(y \mid x_i)$
        \STATE Evaluate gold rewards $g_{ij} \gets g(x_i,y_{ij})$
        \STATE Score samples $s_{\ell ij} \gets s_\ell(x_i,y_{ij})$, for $\ell=1,\ldots,m$
    \ENDFOR
\ENDFOR

\STATE Initialize $\omega^{(0)} \gets (1,\ldots,1)$

\FOR{$t=0$ to $T-1$}
    \FOR{$i=1$ to $k$}
        \STATE Compute candidate logits
        \[
        a_{ij} \gets
        \sum_{\ell=1}^m \omega_\ell^{(t)} s_{\ell ij} - s_{1ij},
        \quad \text{for } j = 1, \ldots, n
        \]
        \STATE Compute candidate weights
        \[
        \pi_{ij} \gets
        \frac{\exp(a_{ij})}
        {\sum_{q=1}^n \exp(a_{iq})},
        \quad \text{for } j = 1, \ldots, n
        \]
    \ENDFOR
    \STATE Compute objective estimate
    \[
    \widehat{J}(\omega^{(t)})
    \gets
    \frac{1}{k}
    \sum_{i=1}^k \sum_{j=1}^n
    \pi_{ij} g_{ij}
    -
    \lambda \|\omega^{(t)}\|^2
    \]
    \STATE Update $\omega^{(t+1)} \gets \omega^{(t)} +
    \eta \nabla_\omega \widehat{J}(\omega^{(t)})$
\ENDFOR

\RETURN $\omega^{(T)}$
\end{algorithmic}
\end{algorithm}

The motivation for considering regularized reward maximization is that neither the reference model $p$ nor the (proxy) reward model $r$ is perfect.
The ultimate objective is to attain a policy that maximizes a gold reward or closely approximates an ideal distribution, neither of which may be readily available or precisely known.
Intuitively, optimizing toward an inaccurate proxy reward model $r$ may be misleading and counter-productive to the ultimate objective.
This is a well-known phenomenon commonly referred to as \emph{reward hacking}~\citep{pan2022effects}, for which the recent works of~\citep{mroueh2025information, huang2025best, aminian2025best, khalaf2025inference} have provided theoretical characterizations and related mitigation strategies.
In particular, our approach extends upon the HedgeTune concept of~\citep{khalaf2025inference} by utilizing a small amount of calibration samples to optimize the SLOP weights $\omega$.

The ultimate objective (in principle) of our reward hacking mitigation approach is to maximize the expected gold reward, over the choice of SLOP weight parameters, as given by
\begin{align}
\label{eq:slop_weight_obj}
\sup_\omega \mathbb{E}_{x \sim \mathcal{D}, x \sim \pi_\omega^*(y \mid x)} [g(x, y)],
\end{align}
where the expectation is over both sampling $x$ from some input distribution $\mathcal{D}$, and sampling $y$ from the SLOP $\pi_\omega^*$.
We note that the optimization problem of~\eqref{eq:slop_weight_obj} does not always have a maximizer with finite weights, as sometimes the supremum reward may only be approached by sharpening the model towards a hard deterministic output with diverging weights.
See Appendix~\ref{app:optimal_SLOP_remarks} for further discussion.

In practice, we generally cannot arbitrarily evaluate the gold reward, which limits us to empirically estimate the objective in~\eqref{eq:slop_weight_obj} via a small number of calibration samples.
We assume that we are given $k$ input samples, $x_1, \ldots, x_k$, drawn \textit{i.i.d.} from the input distribution $\mathcal{D}$, and that for each $x_i$, we can query the gold reward for up to $n$ samples.
In practice, for situations with a verifiable reward, instead of querying the gold model up to $kn$ times, this assumption may instead be realized by just having the gold reference responses corresponding to the $k$ input samples (e.g., for math problems, responses can be checked against reference solutions to assign a correctness reward).
Algorithm~\ref{alg:calibrate_slop} uses these calibration samples to calculate the following empirical estimate of the objective in~\eqref{eq:slop_weight_obj}, upon which it performs gradient ascent, as given by
\begin{align}
\label{eq:calibration_objective}
\widehat{J}(\omega) =
\frac{1}{k} \sum_{i=1}^k \sum_{j=1}^n
\widehat{\pi}_{\omega,n}(y_{ij}\mid x_i)  g(x_i,y_{ij})
- \lambda \| \omega \|^2,
\end{align}
where $\lambda \geq 0$ is optionally used to regularize weight optimization and avoid divergence towards fully sharpened hard outputs.
Given calibrated weights $\omega$, we can either approximately sample the SLOP distribution via the method given by~\eqref{eq:SLOP_SBoN}, or instead employ a hard selection that simply picks the candidate $y_j$ that maximizes the weighted SLOP score $\sum_i \omega_i s_i(x, y_j) - s_1(x, y_j)$.

\subsection{Optimal weights for jointly Gaussian model scores}
\label{sec:gaussian_model}

Consider the binary output scenario, with $\mathcal{Y} = \{0, 1\}$, while $\mathcal{X}$ remains some arbitrary, large set of possible inputs (e.g., all questions that seek a verifiable binary response).
For any $x \in \mathcal{X}$, we let $y_x$ denote the corresponding correct response, and the gold reward is given by $g(x, y) = 1 - | y - y_x |$.
Thus, the gold reward maximization objective is to maximize the probability of sampling the correct response.
For each model $p_i$ and $x \in \mathcal{X}$, let the log-posterior ratio (evaluated on the correct response) be denoted by
\begin{align}
L_i(x) := \log \frac{p_i(y_x \mid x)}{p_i(1 - y_x \mid x)}.
\end{align}
For this illustrative example, we assume that, over $x \sim \mathcal{D}$, $L(x) := (L_1(x), \ldots, L_m(x))$ are jointly Gaussian with mean vector $\bar \mu$ and positive-definite covariance $\Sigma$, to depict a pool of models with potentially varying and correlated reliability.
Let the weighted sum of these ratios be denoted by $Z(x) := \omega^T L(x)$, which is a scalar Gaussian with mean $\mu = \omega^T \bar \mu$ and variance $\sigma^2 = \omega^T \Sigma \omega$.
Classical detection theory~\citep{kay1998detection} establishes that setting $\omega = c \Sigma^{-1} \bar \mu$, for any $c > 0$, maximizes the detection margin $(\mu / \sigma) = \sqrt{\bar \mu^T \Sigma^{-1} \bar \mu}$.
Note that the same analysis and result also applies to the heavier-tailed Cauchy distribution, as discussed in Appendix~\ref{app:cauchy}.

In this binary example, the SLOP evaluated on the correct solution reduces to $\pi_\omega^*(y_x \mid x) = \rho(Z(x))$, where $\rho(z) := 1 / (1 + \exp(-z))$ is the sigmoid function.
Thus, the expected gold reward is given by
\begin{align}
\mathbb{E}_{x \sim \mathcal{D}, y \sim \pi_\omega^*(y \mid x)} [g(x, y)]
= \mathbb{E}_{x \sim \mathcal{D}} [\pi_\omega^*(y_x \mid x)]
= \mathbb{E}_{Z \sim \mathcal{N}(\mu, \sigma^2)} [\rho(Z)] \approx \Phi \left(\frac{\mu}{\sqrt{\frac{8}{\pi} + \sigma^2}} \right),
\end{align}
where $\Phi$ is the probit function (standard normal CDF) and this classical approximation is seen in~\citep{spiegelhalter1990sequential, mackay1992evidence}, with further recent analysis in~\citep{mucsanyi2025rethinking}.
Under this approximation, and setting $\omega = c \Sigma^{-1} \bar \mu$, the expected gold reward is given as
\begin{align}
\mathbb{E}_{x \sim \mathcal{D}, y \sim \pi_\omega^*} [g(x, y)] \longrightarrow \Phi \left(\sqrt{ \bar \mu^T \Sigma^{-1} \bar \mu} \right),
\quad \text{as } c \longrightarrow \infty.
\end{align}
Note that saturating the weights is necessary to maximize accuracy, since we are sampling from $\pi_\omega^*$, which only approaches a hard decision as $c \to \infty$.
However, instead of sampling, if we make a hard decision directly as $\arg \max_y \pi_\omega^*(y \mid x)$, with any fixed $c > 0$, which is equivalent to the optimal log-likelihood ratio test, then the accuracy would be exactly $\Pr[Z > 0] = \Phi \left(\sqrt{ \bar \mu^T \Sigma^{-1} \bar \mu} \right)$.

\subsubsection{Example impact of diversity vs. correlation}
Consider the simple example with $\bar{\mu}=\mathbf{1}_m$ and $\Sigma=\sigma^2_0 ((1-\eta) I_m + \eta \mathbf{1}_m \mathbf{1}_m^T)$, where $\mathbf{1}_m=[1,\ldots, 1]^T$ is an all-ones vector of size $m$ and $\eta \in [0, 1)$ is the correlation factor.
Then, we have
\begin{align}
    \bar{\mu}^T \Sigma^{-1} \bar{\mu} &=
\frac{m}{\sigma^2_0 (1 + \eta(m-1))},
\end{align}
which demonstrates that with uncorrelated ($\eta=0$) experts, the accuracy $\Phi(\sqrt{m/\sigma^2_0})$ approaches one as $m$ grows.
However, with correlation $\eta > 0$, the accuracy saturates to $\Phi(\sqrt{1/\eta\sigma^2})$ for $m \gg 1$.
Thus, it is beneficial to ensemble diverse experts that have low correlation, and ideally independent error patterns.

\begin{figure}[ht]
    \centering

    \begin{subfigure}{0.55\linewidth}
        \centering
        \includegraphics[width=\linewidth]{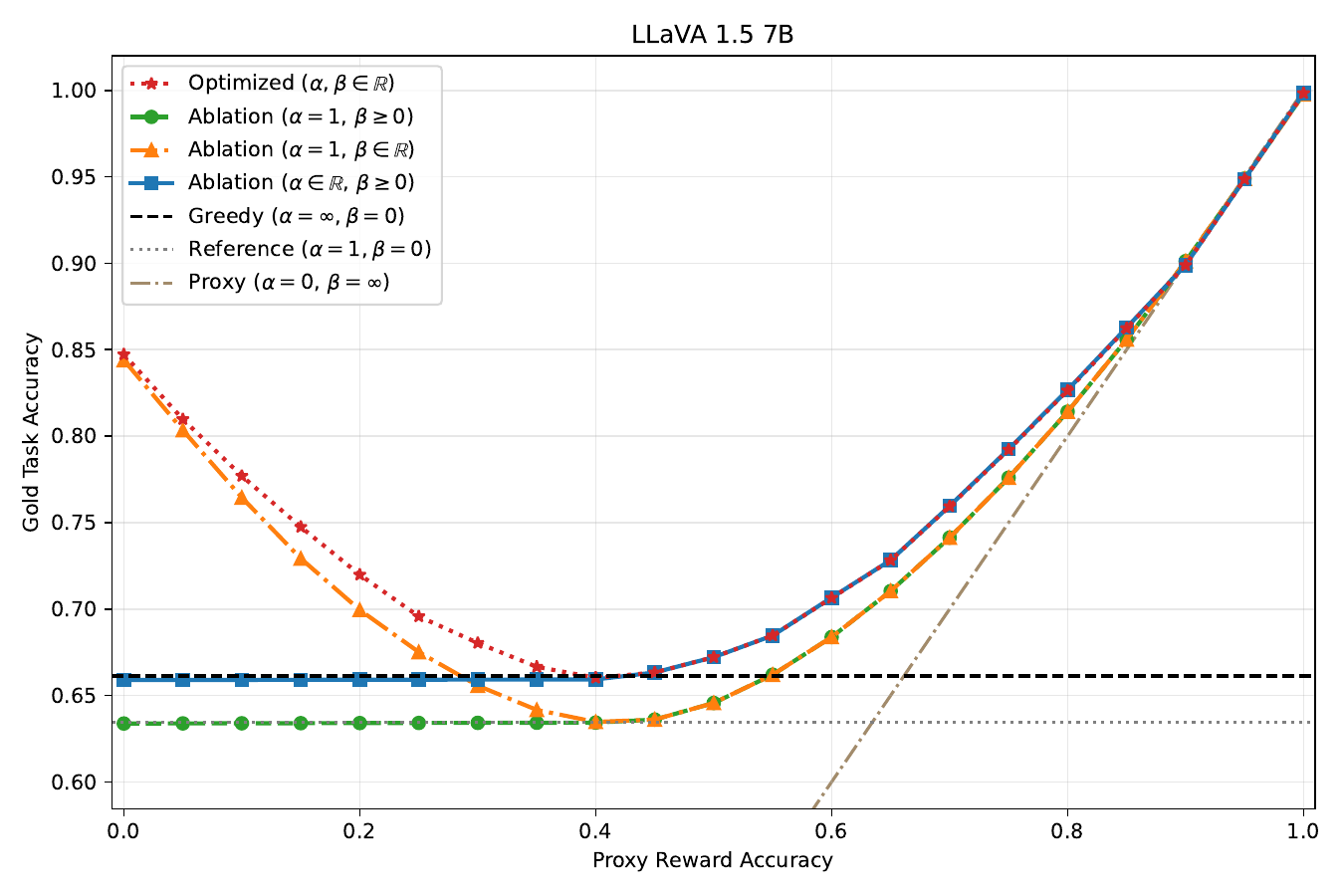}
        \caption{Accuracy}
\label{fig:sqa_llava}
    \end{subfigure}
    \hfill
    \begin{subfigure}{0.44\linewidth}
        \centering
        \includegraphics[width=\linewidth]{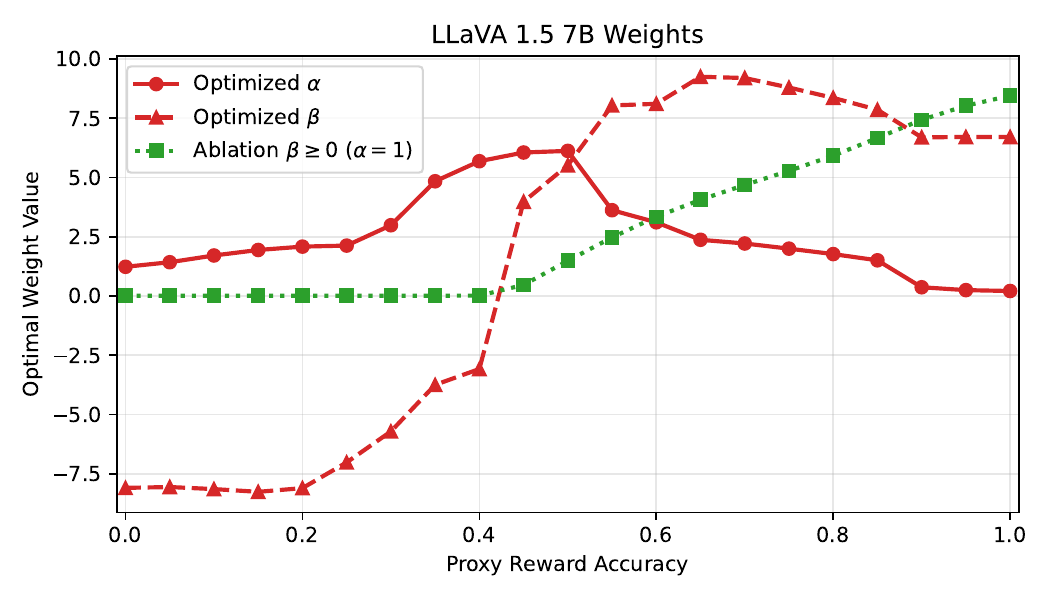}
        \caption{Optimized SLOP weights}
\label{fig:sqa_llava_params}
    \end{subfigure}

    \caption{SLOP with LLaVA-1.5-7B paired with synthetic proxy reward with varying accuracy, evaluated on SQA.}
    \label{fig:sqa_llava_combined}
\end{figure}

\begin{figure}[t]
    \centering

    \begin{subfigure}{0.49\linewidth}
        \centering
        \includegraphics[width=\linewidth]{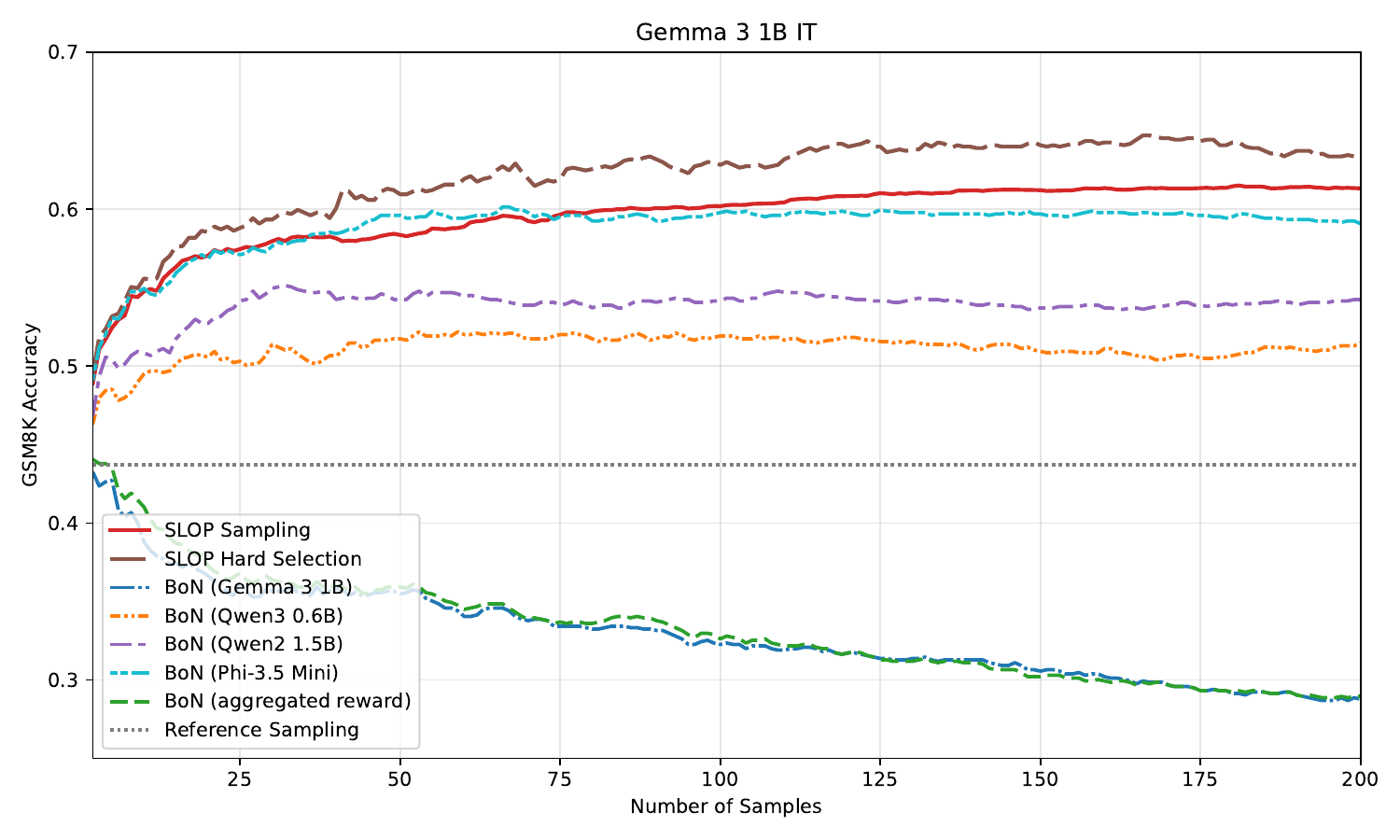}
    \end{subfigure}
    \hfill
    \begin{subfigure}{0.49\linewidth}
        \centering
        \includegraphics[width=\linewidth]{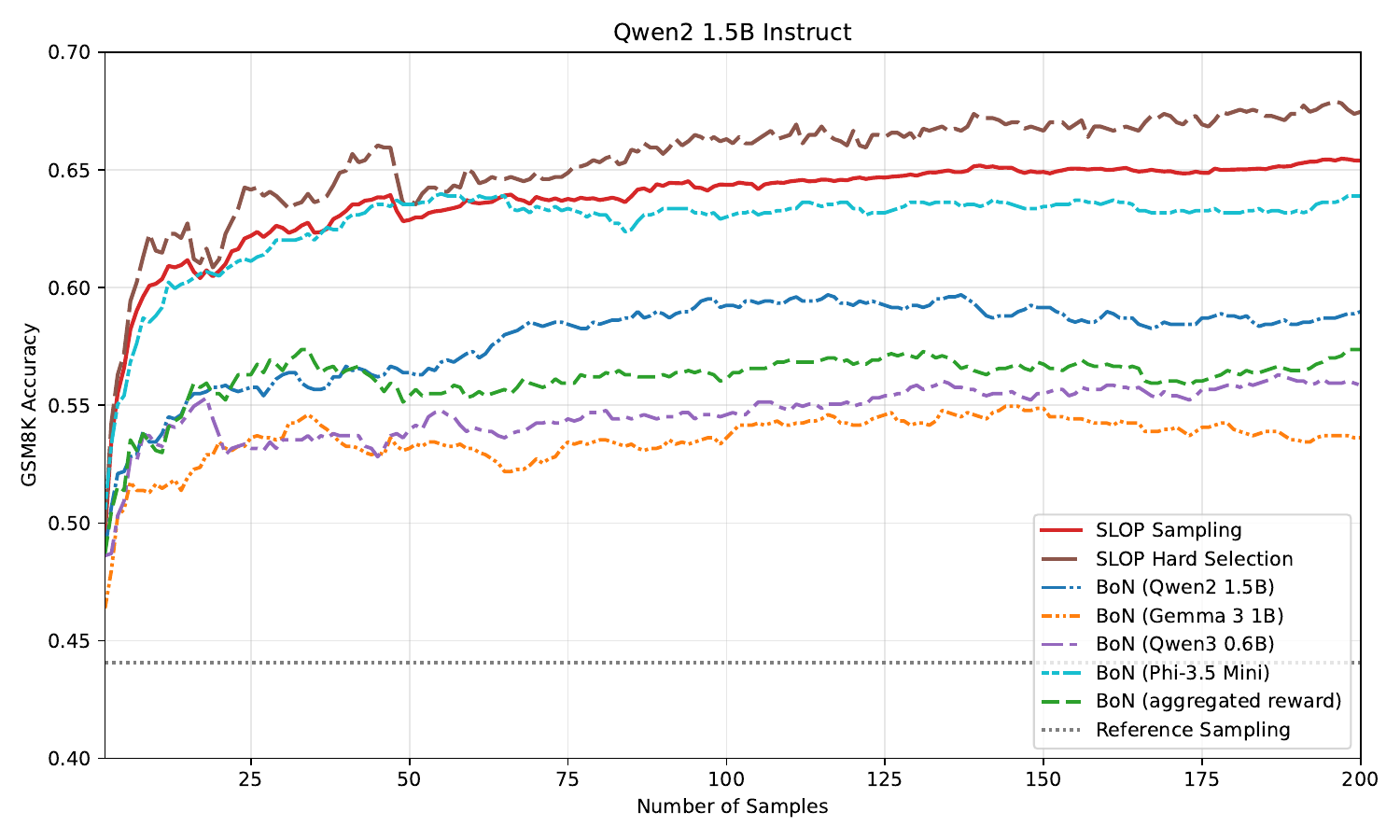}
    \end{subfigure}

    \vspace{0.5em}

    \begin{subfigure}{0.49\linewidth}
        \centering
        \includegraphics[width=\linewidth]{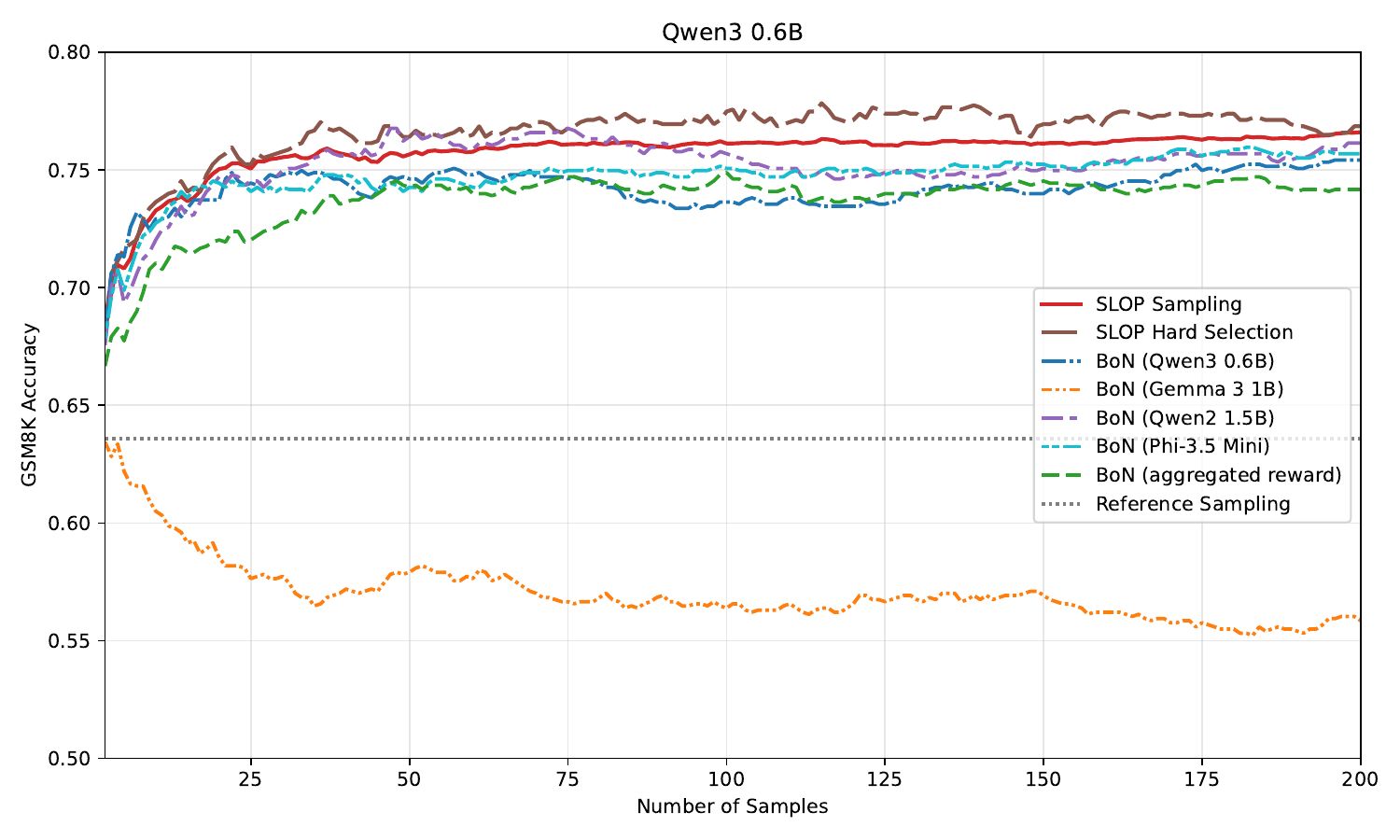}
    \end{subfigure}
    \hfill
    \begin{subfigure}{0.49\linewidth}
        \centering
        \includegraphics[width=\linewidth]{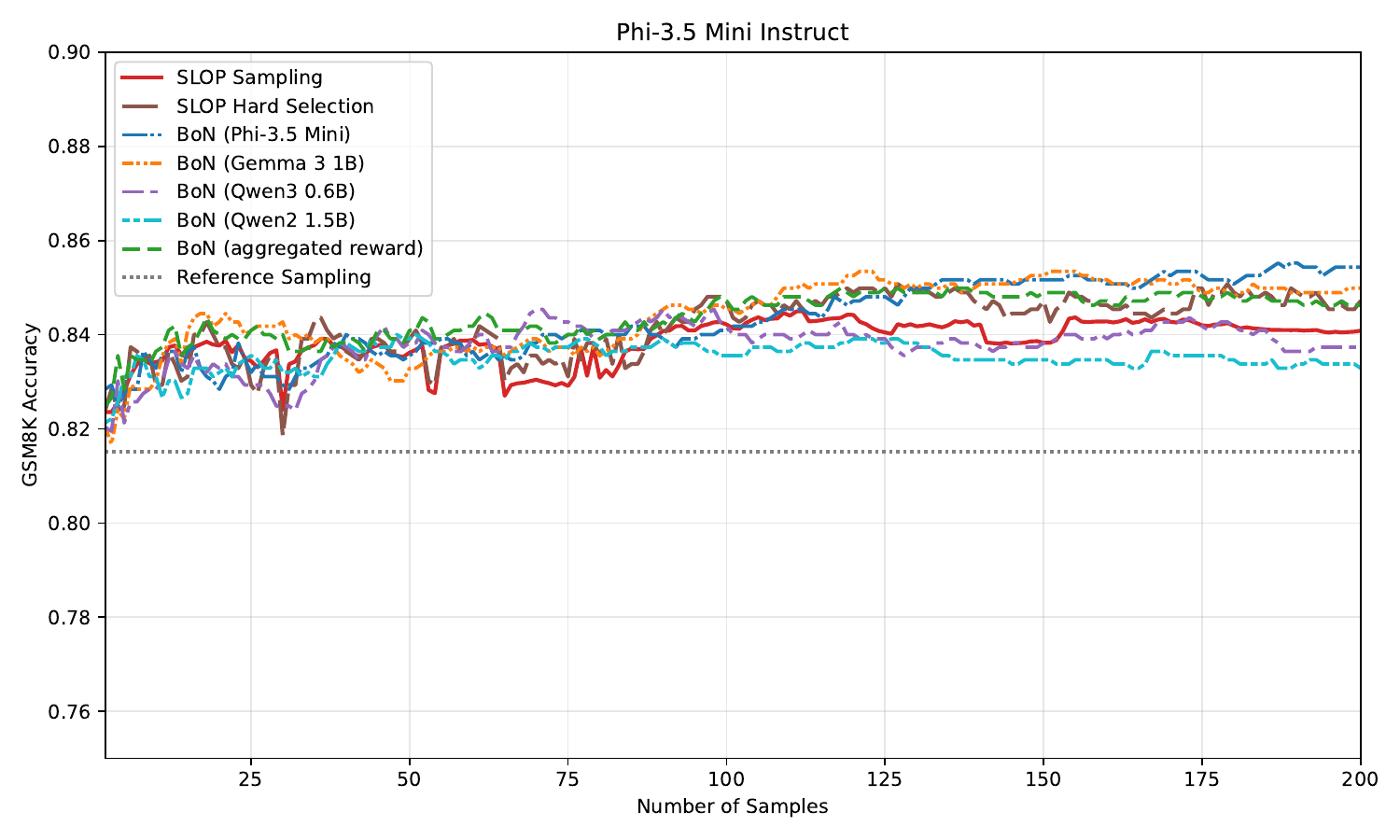}
    \end{subfigure}

    \caption{SLOP (with 4 LLMs) evaluated on GSM8K, with different LLMs as the reference model (indicated by the subfigure titles).}
    \label{fig:gsm8k_all}
\end{figure}

\section{Experimental results}
\label{sec:experiments}

\subsection{Reward-guided visual question answering}
\label{sec:SQA_exp}

This experiment considers the task of multiple-choice, visual question answering, evaluated with the ScienceQA (SQA)~\citep{lu2022learn} benchmark dataset.
The gold reward is one for the correct answer and zero otherwise.
Thus, the objective is to maximize accuracy.
To investigate the impact of imprecise reward models, we take a pretrained VLM as the generative reference model and pair it with a synthetic proxy reward model that may instead reward incorrect answers, at some given error rate.
Since SQA multiple-choice questions contain only two to five answer options, we constrain VLM generation to single-token decoding over the small response set $\mathcal{Y} = \{ \text{A}, \text{B}, \text{C}, \text{D}, \text{E} \}$.
This simplification allows for exact sampling of the tempered-and-tilted (two-expert SLOP) distribution and focuses this controlled experiment on weight optimization for inaccurate proxy reward models.
The $(\alpha, \beta)$ weights are optimized with $200$ calibration samples from the SQA test split, while the remaining $4041$ samples of the SQA test split are held-out for evaluation.
See Appendix~\ref{app:more_SQA} for further details and additional results with the Gemma-3-4B~\citep{team2025gemma} and Qwen3-VL-4B~\citep{bai2025qwen3} VLMs, which yield similar findings.

Figure~\ref{fig:sqa_llava} plots the results for the LLaVa-1.5-7B~\citep{liu2023llava, liu2024improved} VLM paired with a synthetic proxy reward model at varying accuracies.
The ``optimized'' curve is the result of freely optimizing $\alpha, \beta \in \mathbb{R}$, while the ``ablation'' curves constrain with $\alpha = 1$ and/or $\beta \geq 0$.
Two baselines ignore the proxy reward model ($\beta = 0$): sampling the ``reference'' model ($\alpha = 1$, $63.4\%$ accuracy) and ``greedy'' token selection ($\alpha = \infty$, $66.1\%$ accuracy).
On the other hand, ignoring the reference model and purely following the ``proxy'' reward ($\alpha = 0$, $\beta = \infty$) yields the proxy reward accuracy, which works well for accurate proxies, but quickly degrades (along the slope-one line, falling off the plot) for inaccurate proxies, which depicts the effect of hacking with a faulty reward model.

The SLOP curves generally dominate the performance of the individual reference and proxy models, by effectively combining the two.
For the ablations with $\beta \geq 0$, performance gracefully degrades back to the baselines as the proxy reward model becomes more unreliable.
The optimized and ablation curves, with unconstrained $\beta$, yield increased performance for low proxy reward accuracies, since below $40\%$ accuracy (which is roughly the random guessing accuracy for SQA), the inaccurate proxy more reliably indicates an incorrect answer, providing a useful signal that can be exploited with weight $\beta < 0$.
This can be seen in Figure~\ref{fig:sqa_llava_params}, which plots the optimized $(\alpha, \beta)$ weights across varying proxy reward accuracies, showing how the confidence weight for each model correspondingly varies.
At high proxy accuracies, the optimal $\alpha$ is close to zero, as following the proxy suffices.
However, at middling proxy accuracies, $\alpha$ is larger, showing more reliance on the confidence of the reference model.
For comparison, it also plots the optimal $\beta$ for the ablation (with fixed $\alpha = 1$ and $\beta \geq 0$), which simply falls back to $0$ at $40\%$ proxy reward accuracy or below.
We note that the ablation with fixed $\alpha = 1$ and $\beta \geq 0$ is essentially the HedgeTune method~\citep{khalaf2025inference}.

\subsection{Math reasoning with generative model pooling}
\label{sec:gsm8k_exp}

This experiment explores the task of math reasoning, evaluated on
the GSM8K~\citep{cobbe2021training} benchmark, with a SLOP that ensembles four LLMs:
Gemma-3-1B~\citep{team2025gemma}, Qwen2-1.5B~\citep{yang2024qwen2}, Qwen3-0.6B~\citep{yang2025qwen3}, and Phi-3.5-mini~\citep{abdin2024phi3}.
In this task, the input prompt is a grade school math word problem, and responses with the correct numerical solution receive a gold reward of one, and zero otherwise.
Thus, the objective is to maximize accuracy.

We apply the approximate SLOP sampling method of~\eqref{eq:SLOP_SBoN}, with one of the LLMs serving as the reference model to sample up to $200$ responses, and the expert scores given by the token-averaged log-likelihoods of each LLM, i.e., $s_i(x, y) = (1/L_y) \log p(y \mid x)$, where $L_y$ is the number of tokens in the response $y$ and is introduced to avoid length bias.
We apply Algorithm~\ref{alg:calibrate_slop} to calibrate the SLOP weights using $200$ calibration examples from the GSM8K test split, while the remaining $1119$ samples are held-out for evaluation.

Figure~\ref{fig:gsm8k_all} plots performance, while varying the number of candidate samples $n$, for each LLM selected as the reference model to generate candidates.
The BoN baselines select from the candidates using the individual LLM expert scores or the simple sum of the expert scores as the proxy reward.
SLOP sampling tends to outperform the baselines, with more stable, monotonic improvement as the number of samples increase.
We also evaluate ``hard selection'' of the candidate with the highest weighted SLOP score, which further improves performance over SLOP sampling.
An exceptional case occurs when the strongest LLM (Phi-3.5) is the reference model, where all methods yield only slight improvement and fall within a few percentage points of each other. 
See Appendix~\ref{app:more_gsm8k} for further details, with the figures reproduced in larger format as Figures~\ref{fig:gsm8k_gemma3},~\ref{fig:gsm8k_qwen2},~\ref{fig:gsm8k_qwen3}, and~\ref{fig:gsm8k_phi3.5}.

\subsection{SLOP weight selection from score statistics}
\label{sec:gsm8k_covariance}

This experiment continues with the same math reasoning setting of Section~\ref{sec:gsm8k_exp}, and explores whether effective SLOP weights can be estimated from expert score statistics rather than optimization, by applying the insights and Gaussian model assumption of Section~\ref{sec:gaussian_model}.
Since this setting is not limited to binary responses, we instead apply this assumption to model the relative scores for any pair of correct and incorrect responses.
We estimate the mean $\bar{\mu}$ and covariance $\Sigma$ of these relative scores from all such sampled pairs across the $200$ calibration questions.
Then, we assign weights as $\omega = \Sigma^{-1} \bar \mu$, which we call the inverse covariance SLOP.
To avoid a potentially ill-conditioned inverse, we also consider the diagonalized covariance SLOP, which first drops the off-diagonal terms of the covariance matrix before performing the inverse.
As another ablation, we also compare with the simple fixed weight strategy with $\omega = \mathbf{1}$.

Table~\ref{tab:gsm8k_results} summarizes the average performance of these three ablations for SLOP weight selection, alongside the baseline reference model sampling, SLOP with calibrated weights, and the top-performing BoN baseline.
Performance plots are found in Figures~\ref{fig:gsm8k_gemma3_ablation},~\ref{fig:gsm8k_qwen2_ablation},~\ref{fig:gsm8k_qwen3_ablation}, and~\ref{fig:gsm8k_phi3.5_ablation} in Appendix~\ref{app:more_gsm8k}.
Table~\ref{tab:gsm8k_weights} lists the optimized SLOP weights, along with the cosine similarity scores between these weights and each of the three ablation methods for assigning weights, and also reports the condition number of the relative score covariance matrix.
Generally, the covariance-based methods improve upon the reference model sampling baseline (except the inverse covariance case with Gemma-3, which seems to suffer from ill-conditioning), and can be competitive against the calibrated SLOP.

\begin{table}[ht]
\centering
\caption{Average accuracy (over $n=20, \ldots, 200$) evaluated on GSM8K for each reference model:
Baseline reference sampling, SLOP sampling, SLOP Hard selection, top-performing BoN, Fixed weights ($\omega = \mathbf{1}$), Diagonal Covariance hard selection, and Inverse Covariance hard selection.  
Error intervals are $\pm \sigma$, while the
\textbf{first} and \underline{second} are highlighted.}
\label{tab:gsm8k_results}
\small
\setlength{\tabcolsep}{3.5pt}
\begin{tabular}{@{}lccccccc@{}}
\toprule
\textbf{Ref. Model} &
\textbf{Baseline \%} &
\textbf{SLOP \%} & 
\textbf{SLOP$_\text{Hard}$ \%} & 
\textbf{BoN* \%} &
\textbf{Fixed \%} & 
\textbf{DiagCov \%} & 
\textbf{InvCov \%} \\
\midrule
Gemma-3 & 
  $43.70\,{\scriptstyle \pm 0.10}$ & 
  $\underline{60.13}\,{\scriptstyle \pm 1.27}$ & 
  $\mathbf{62.76}\,{\scriptstyle \pm 1.60}$ & 
  $59.37\,{\scriptstyle \pm 0.67}$ &
  $55.50\,{\scriptstyle \pm 0.42}$ &
  $47.06\,{\scriptstyle \pm 0.30}$ &
  $43.44\,{\scriptstyle \pm 0.43}$ \\
Qwen2 & 
  $44.08\,{\scriptstyle \pm 0.10}$ & 
  $\underline{64.23}\,{\scriptstyle \pm 0.95}$ & 
  $\mathbf{65.96}\,{\scriptstyle \pm 1.30}$ & 
  $63.21\,{\scriptstyle \pm 0.62}$ & 
  $55.96\,{\scriptstyle \pm 0.49}$ &
  $57.15\,{\scriptstyle \pm 0.63}$ &
  $62.91\,{\scriptstyle \pm 0.84}$ \\
Qwen3 & 
  $63.58\,{\scriptstyle \pm 0.10}$ & 
  $76.05\,{\scriptstyle \pm 0.31}$ & 
  $\mathbf{76.93}\,{\scriptstyle \pm 0.51}$ & 
  $75.51\,{\scriptstyle \pm 0.61}$ &
  $71.55\,{\scriptstyle \pm 0.45}$ &
  $\underline{76.49}\,{\scriptstyle \pm 0.33}$ &
  $76.16\,{\scriptstyle \pm 0.50}$ \\
Phi-3.5 & 
  $81.51\,{\scriptstyle \pm 0.08}$ & 
  $83.91\,{\scriptstyle \pm 0.42}$ & 
  $\underline{84.27}\,{\scriptstyle \pm 0.61}$ & 
  $\mathbf{84.47}\,{\scriptstyle \pm 0.41}$ &
  $\underline{84.27}\,{\scriptstyle \pm 0.51}$ &
  $83.77\,{\scriptstyle \pm 0.84}$ &
  $83.67\,{\scriptstyle \pm 0.64}$ \\
\bottomrule
\end{tabular}
\end{table}

\begin{table}[ht]
\centering
\caption{For each reference model:
optimized SLOP weights $\omega$, cosine similarity scores of weight selection ablations, and covariance condition number.}
\label{tab:gsm8k_weights}
\small
\setlength{\tabcolsep}{4pt}
\begin{tabular}{@{}lcccccccc@{}}
\toprule
\textbf{Ref. Model} &
${\omega_{\text{Gemma-3}}}$ &
${\omega_{\text{Qwen2}}}$ &
${\omega_{\text{Qwen3}}}$ &
${\omega_{\Phi\text{-3.5}}}$ & 
\textbf{Fixed} &
\textbf{DiagCov} & 
\textbf{InvCov} & 
\textbf{$\mathbf{Cond}(\Sigma)$} \\
\midrule
Gemma-3-1B & $-0.038$ & $2.55$  & $10.2$ & $33.9$ & 
  $0.657$ &
  $0.824$ &
  $0.815$ &
  $4256.9$ \\
Qwen2-1.5B & $-9.44$  & $-2.36$ & $1.94$ & $35.0$ &
  $0.346$ &
  $0.244$ &
  $0.422$ &
  $64.40$ \\
Qwen3-0.6B & $-3.55$  & $12.7$  & $17.0$ & $17.1$ &
  $0.787$ &
  $0.970$ &
  $0.938$ &
  $42.32$ \\
Phi-3.5-mini & $4.65$   & $6.81$  & $6.12$ & $8.82$ & 
  $0.975$ &
  $0.962$ &
  $0.728$ &
  $95.33$ \\
\bottomrule
\end{tabular}
\end{table}

\section{Discussion}
\label{sec:discussion}

We develop inference-time alignment methods intended to improve robustness to reward hacking and to better combine signals from multiple generative or reward models.
Potential positive impacts include more effective deployment of foundational models, especially when proxy rewards are imperfect.
Potential negative impacts include misuse of inference-time methods to make models better at satisfying flawed, biased or harmful objectives, or amplification of undesirable model behaviors when the calibration signal is misaligned.
Misuse risks are challenging to fully mitigate, however we emphasize the importance of calibration with gold or verifiable rewards.

\subsection{Limitations}
\label{sec:limitations}

A limitation of our work is the somewhat narrow set of experiments that we conducted.
As alignment is a fundamental topic with broad applications, there is certainly much more to be investigated.
Larger and more capable models should be considered, as well as different modalities, beyond LLMs and VLMs, such as multimedia generation or vision-language-action models for embodied AI applications.
Our empirical study was limited to question answering and math reasoning tasks, however broader datasets, benchmarks, and tasks should also be explored.
Besides tasks with verifiable rewards, alignment for other objectives, such as safety and trust (e.g., robustness to jailbreaking attacks, or reducing hallucinations and toxicity) are also interesting directions to be evaluated.

However, our aim was mainly to provide theoretic insights that inspire new developments, which we tested with focused and controlled experiments.
We were also limited by the modest compute resources (a single desktop GPU) available for this study, as discussed in Appendix~\ref{app:compute}.

\subsection{Concluding remarks and future work}

We provide a generalized inference-time alignment framework that flexibly ensembles multiple generative and/or reward models.
In this perspective, continual adaptation and reward-hacking mitigation take the form of calibrating the weights of this ensemble, which depends not only on individual expert quality, but also on the correlation among experts.
Our future work will broaden the empirical evaluation beyond our current proof-of-concept experiments, in new directions, as discussed above.
We will further investigate the roles of correlation and diversity in expert ensembles, and how they can be cultivated to improve performance.

\bibliographystyle{plainnat}
\bibliography{refs}


\appendix

\section{Proof of Corollary~\ref{co:temper_tilt}}

This can be shown as a corollary of the $\alpha = 1$ special case considered by~\citet{korbak2022rl}.
However, for completeness, we give the straightforward derivation, starting from~\eqref{eq:tempered_reward}, 
\begin{align}
\pi_{\alpha, \beta}^*(y \mid x) &=
\arg \max_\pi \mathbb{E}_{Y \sim \pi(y \mid x)}
  \big[ \beta r(x,Y) + (\alpha - 1) \log p(Y \mid x) \big]
  - \mathrm{KL}\big( \pi(y \mid x) \| p(y \mid x) \big) \\
&= \arg \max_\pi \mathbb{E}_{Y \sim \pi(y \mid x)}
  \left[ \beta r(x,Y) + \log \frac{p(Y \mid x)^\alpha}{C_{\alpha, x} \pi(Y \mid x)} \right]
  + \log C_{\alpha, x} \label{eq:temper_equivalent} \\
&= \arg \max_\pi \mathbb{E}_{Y \sim \pi(y \mid x)}
  \left[ \log \frac{p(Y \mid x)^\alpha \exp(\beta r(x,Y))}{C_{\alpha, \beta, x} \pi(Y \mid x)} \right] + \log C_{\alpha, \beta, x} \\
&= \arg \min_\pi \mathrm{KL} \left(
\pi (y \mid x) \bigg \|
\frac{p(y \mid x)^\alpha \exp(\beta r(x, y))}{C_{\alpha, \beta, x}} \right),
\end{align}
where \eqref{eq:temper_equivalent} is equivalent to \eqref{eq:alt_tempered_reward}, and the final KL minimization is achieved by \eqref{eq:temper_tilted}.

\section{Optimizing SLOP weights for reward maximization}
\label{app:optimal_SLOP_remarks}

We remark that the optimization problem of~\eqref{eq:slop_weight_obj} does not always have a maximizer with finite weights.
To illustrate, consider the following simple binary example, where $\mathcal{X} = \mathcal{Y} = \{0, 1\}$, with $g(x, y) = 2 |y - x|$, input distribution $\mathcal{D} = \mathrm{Bernoulli}(0.5)$, and a single expert ($m=1$) SLOP,
\begin{align}
\pi_\omega^*(y \mid x) = \mathrm{softmax}(\omega_1 \log p_1(y \mid x))
= \rho \left( \omega_1 \log \frac{p_1(y \mid x)}{p_1(1-y \mid x)} \right),
\end{align}
where $\rho(z) := 1 / (1 + \exp(-z))$ is the sigmoid function, and we assume that $p_1(y \mid x) \in (0, 1)$ for all $x, y \in \{0, 1\}$.
For this example, the expected reward is given by
\begin{align}
\mathbb{E}_{x \sim \mathcal{D}, y \sim \pi_\omega^*(y \mid x)} [g(x, y)] = \rho \left( \omega_1 \log \frac{p_1(1 \mid 0)}{p_1(0 \mid 0)} \right) + \rho \left( \omega_1 \log \frac{p_1(0 \mid 1)}{p_1(1 \mid 1)} \right),
\end{align}
as a function of two log-likelihood ratios.
If both ratios are zero, then the expected reward equals one for any $\omega_1$.
If one ratio is strictly positive and the other is strictly negative, then a unique maximum exists for some finite $\omega_1$.
However, if both ratios have the same sign, or if only one of the ratios is zero, then the supremum of the expected reward (which equals two, if they are both non-zero, or $1.5$, if one ratio is zero) is only approached as $\omega_1 \to \infty$ (if the ratios are non-negative) or $\omega_1 \to -\infty$ (if the ratios are non-positive).
We can interpret these latter cases as the model $p_1$ being universally aligned or anti-aligned (across all inputs $x \in \{0, 1\}$), and divergence of the weight $\omega_1$ to $\pm \infty$ has the effect of sharpening the model towards a hard (deterministic) output.

\section{Generalizing the Gaussian score assumption to the Cauchy distribution}
\label{app:cauchy}

To consider a heavier-tailed distribution than Gaussian,
suppose the log-likelihood ratio (LLR) vector $L(x)$ follows the multivariate Cauchy distribution:
$L(x) \sim \mathrm{Cauchy}(\bar{\mu}, \Sigma)$, where $\bar{\mu}$ is a location vector and $\Sigma$ is a scatter matrix.
The weighted sum of the LLRs $Z(x):=\omega^T L(x)$ follows the univariate Cauchy distribution: $Z(x)\sim \mathrm{Cauchy}(\mu, \sigma)$, where $\mu=\omega^T \bar{\mu}$ is the location and $\sigma=\sqrt{\omega^T \Sigma \omega}$ is the scale.
We aim to maximize the detection margin $(\mu/\sigma)$ as the accuracy is given by the complementary CDF of the Cauchy distribution: 
\begin{align}
    \Pr(Z(x) > 0) =
    \frac{1}{2} + \frac{1}{\pi}
    \arctan\left(\frac{\mu}{\sigma} \right).
\end{align}
The optimal weights that maximize the detection margin ($\mu/\sigma$) is given by $\omega = c\Sigma^{-1} \bar{\mu}$, which is the same as the Gaussian case.

\begin{figure*}[ht]
\begin{center}
\includegraphics[width=0.9\textwidth]{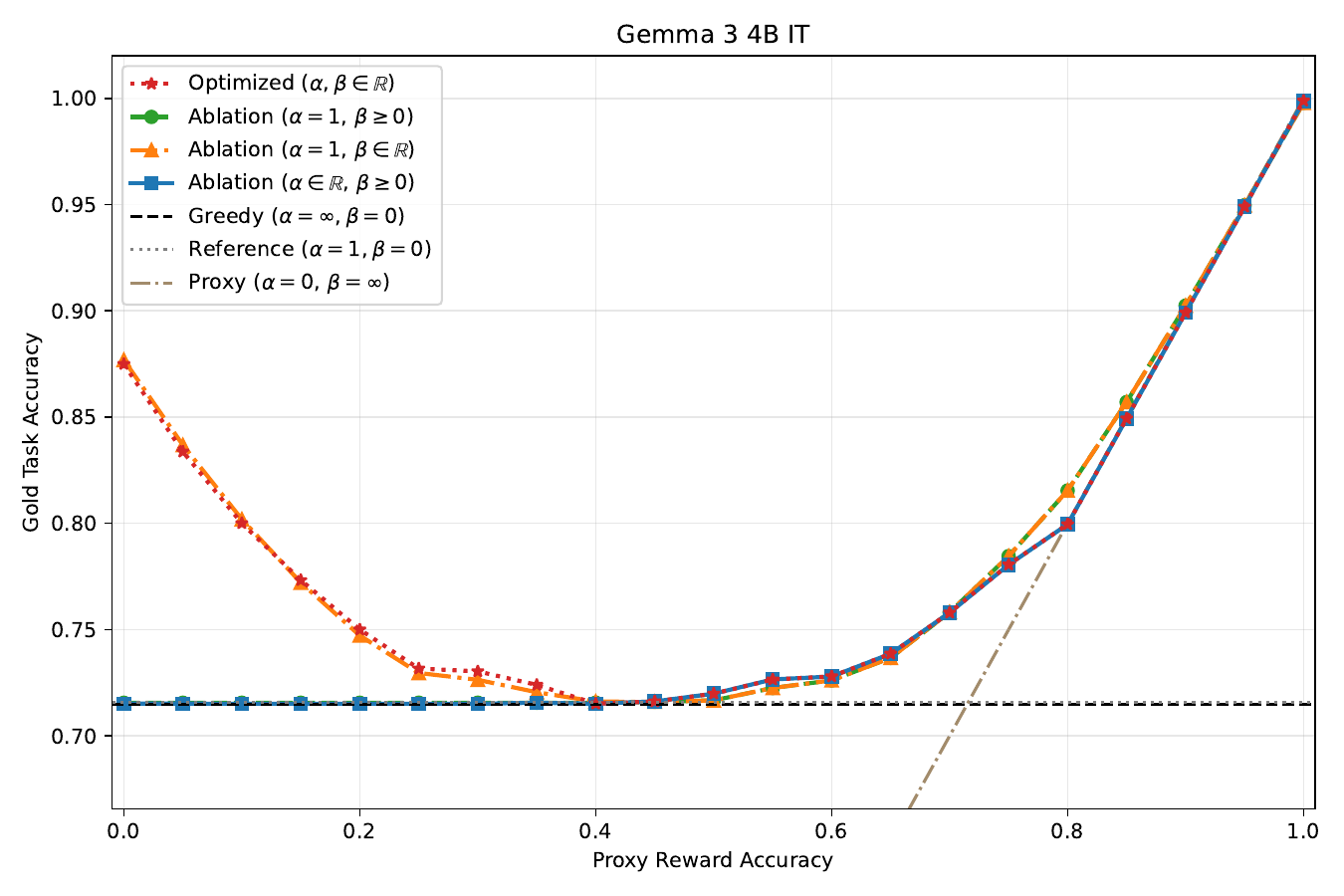}
\caption{SLOP with Gemma-3-4B paired with synthetic proxy reward with varying accuracy, evaluated on SQA.}
\label{fig:sqa_gemma}
\end{center}
\end{figure*}

\begin{figure*}[ht]
\begin{center}
\includegraphics[width=0.9\textwidth]{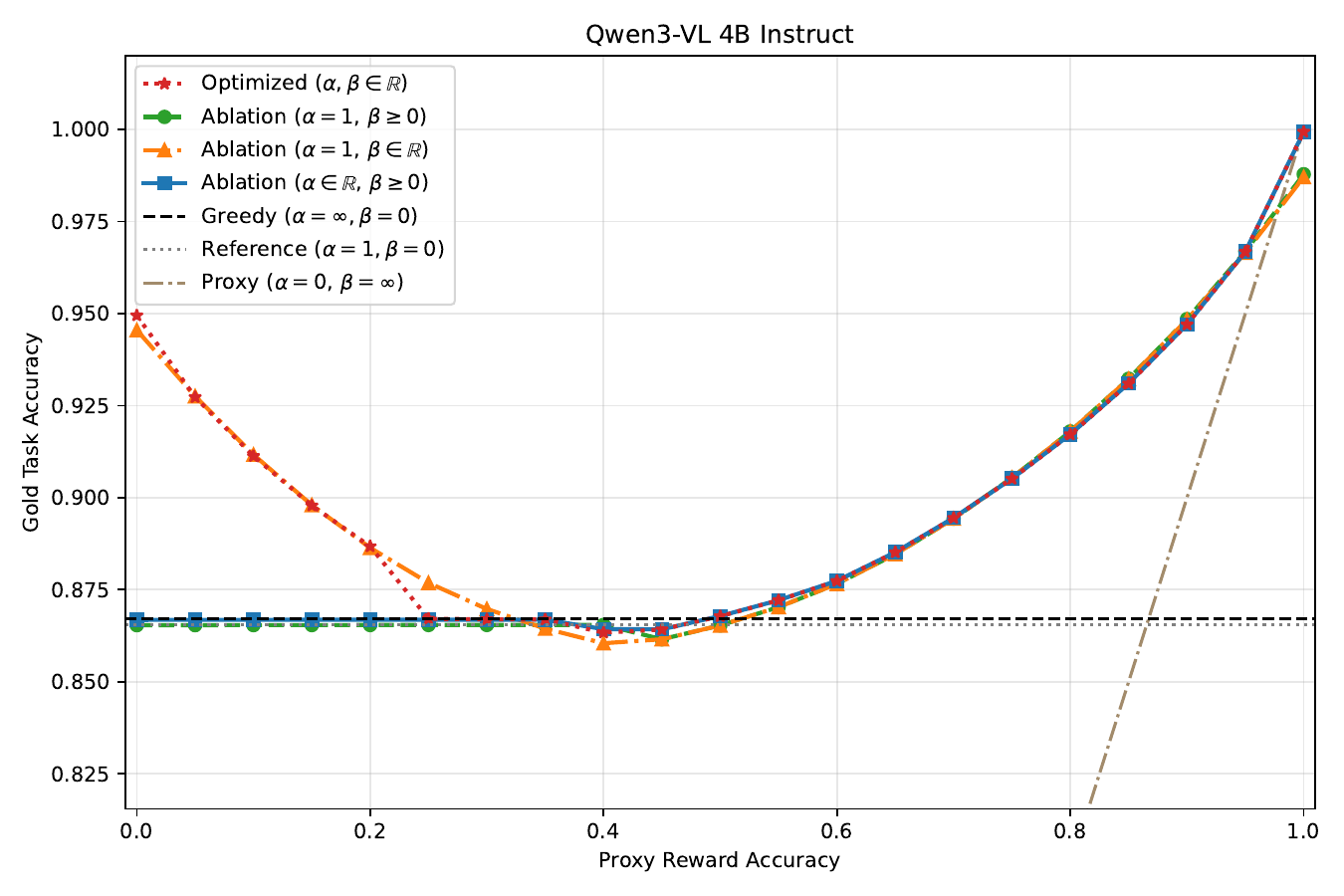}
\caption{SLOP with Qwen3-VL-4B paired with synthetic proxy reward with varying accuracy, evaluated on SQA.}
\label{fig:sqa_qwen}
\end{center}
\end{figure*}

\begin{figure}[ht]
    \centering

    \begin{subfigure}{0.48\linewidth}
        \centering
        \includegraphics[width=\linewidth]{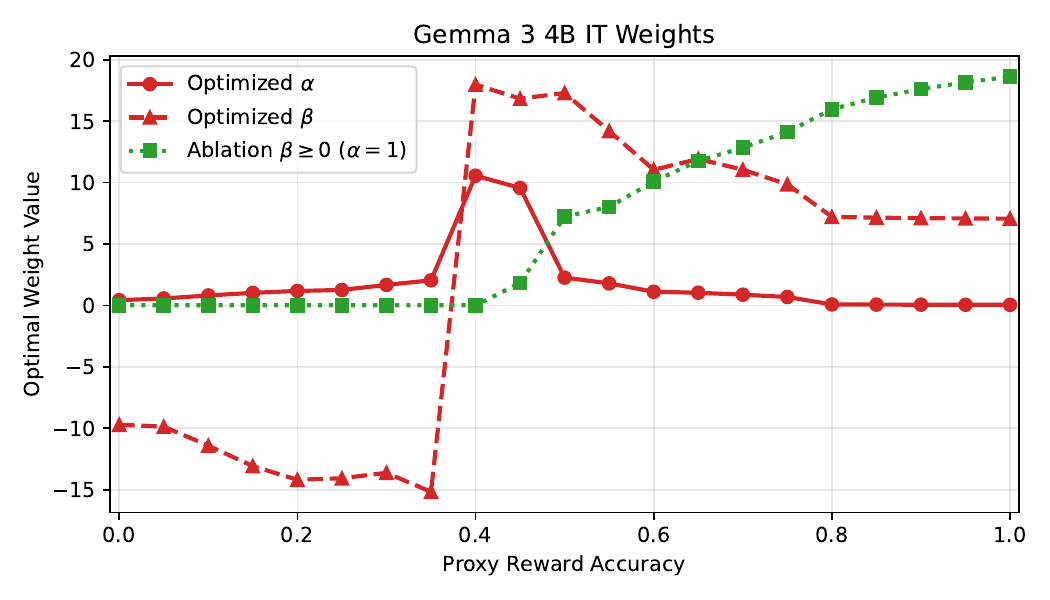}
    \end{subfigure}
    \hfill
    \begin{subfigure}{0.48\linewidth}
        \centering
        \includegraphics[width=\linewidth]{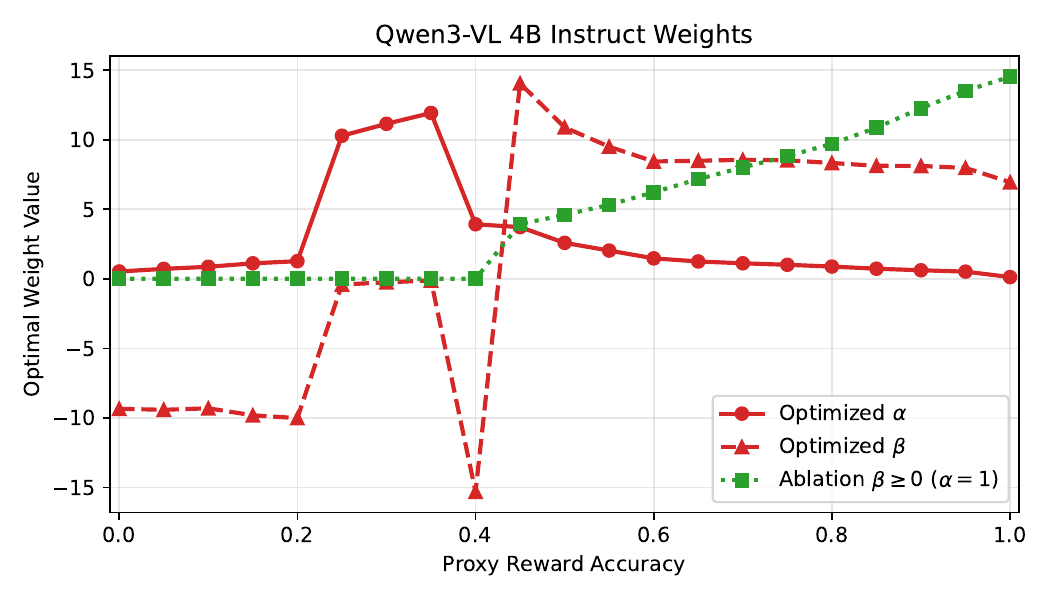}
    \end{subfigure}

    \caption{Optimized (two-expert) SLOP weights for SQA.}
    \label{fig:SQA_more_params}
\end{figure}

\section{Visual question answering experiment details and additional results}
\label{app:more_SQA}

Figures~\ref{fig:sqa_gemma} and~\ref{fig:sqa_qwen} plot additional performance results for 
Gemma-3-4B~\citep{team2025gemma} and Qwen3-VL-4B~\citep{bai2025qwen3}, while Figure~\ref{fig:SQA_more_params} plots their corresponding optimized SLOP weights.

The VLM is prompted with image and question text pairs from the SQA dataset, preceded by the following system prompt:

\verb|    A chat between a curious user and an artificial intelligence assistant.|
\verb|    The assistant gives helpful answers to the user's multiple-choice|
\verb|    questions, responding with the letter of the correct choice only.|

Weight optimization is performed for $T=500$ steps, with learning rate $\eta = 0.05$ and weight decay parameter $\lambda = 0$, and by using Adam~\citep{kingma2014adam}, with $\beta_1 = 0.9$ and $\beta_2 = 0.999$.

For this experiment, we restrict sampling to only decoding a single-token over the small subset of possible answers, which allows for exactly sampling the corresponding SLOP (instead of approximate sampling via SBoN).
This simplifies the calibration objective in~\eqref{eq:calibration_objective} to
\begin{align}
\widehat{J}(\alpha, \beta) =
\frac{1}{k} \sum_{i=1}^k \sum_{y \in \mathcal{Y}(x_k)}
\mathrm{softmax}\big( \alpha \log p(y \mid x_k) + \beta r(x_k, y) \big) g(x_i, y),
\end{align}
where each inner summation and $\mathrm{softmax}$ is only over $\mathcal{Y}(x_k) \subset \mathcal{Y}$, which denotes the subset of possible answers provided by each multiple-choice question $x_k$.

\begin{figure*}[ht]
\begin{center}
\includegraphics[width=\textwidth]{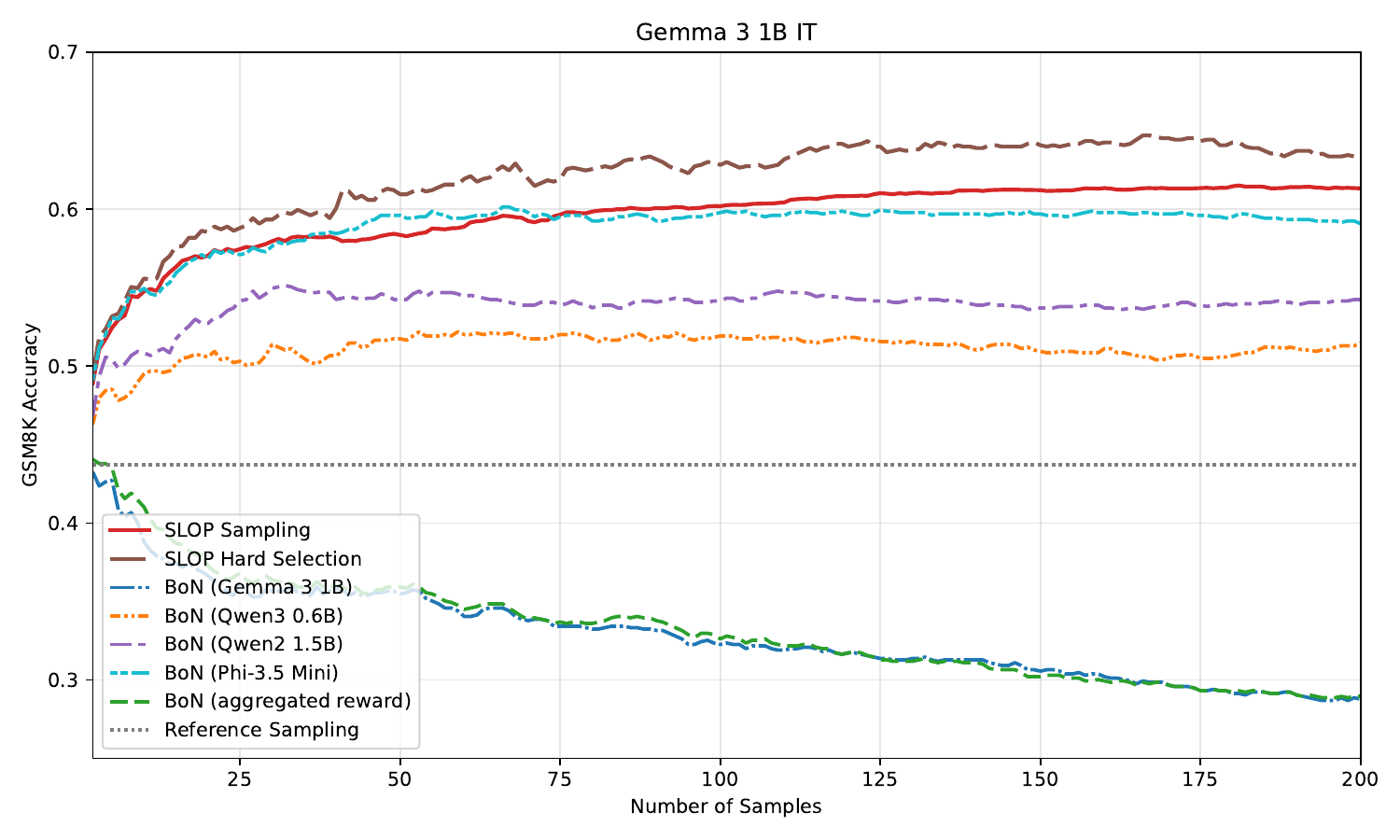}
\caption{SLOP (with 4 LLMs) evaluated on GSM8K, with Gemma-3-1B as the reference model.}
\label{fig:gsm8k_gemma3}
\end{center}
\end{figure*}

\begin{figure*}[ht]
\begin{center}
\includegraphics[width=\textwidth]{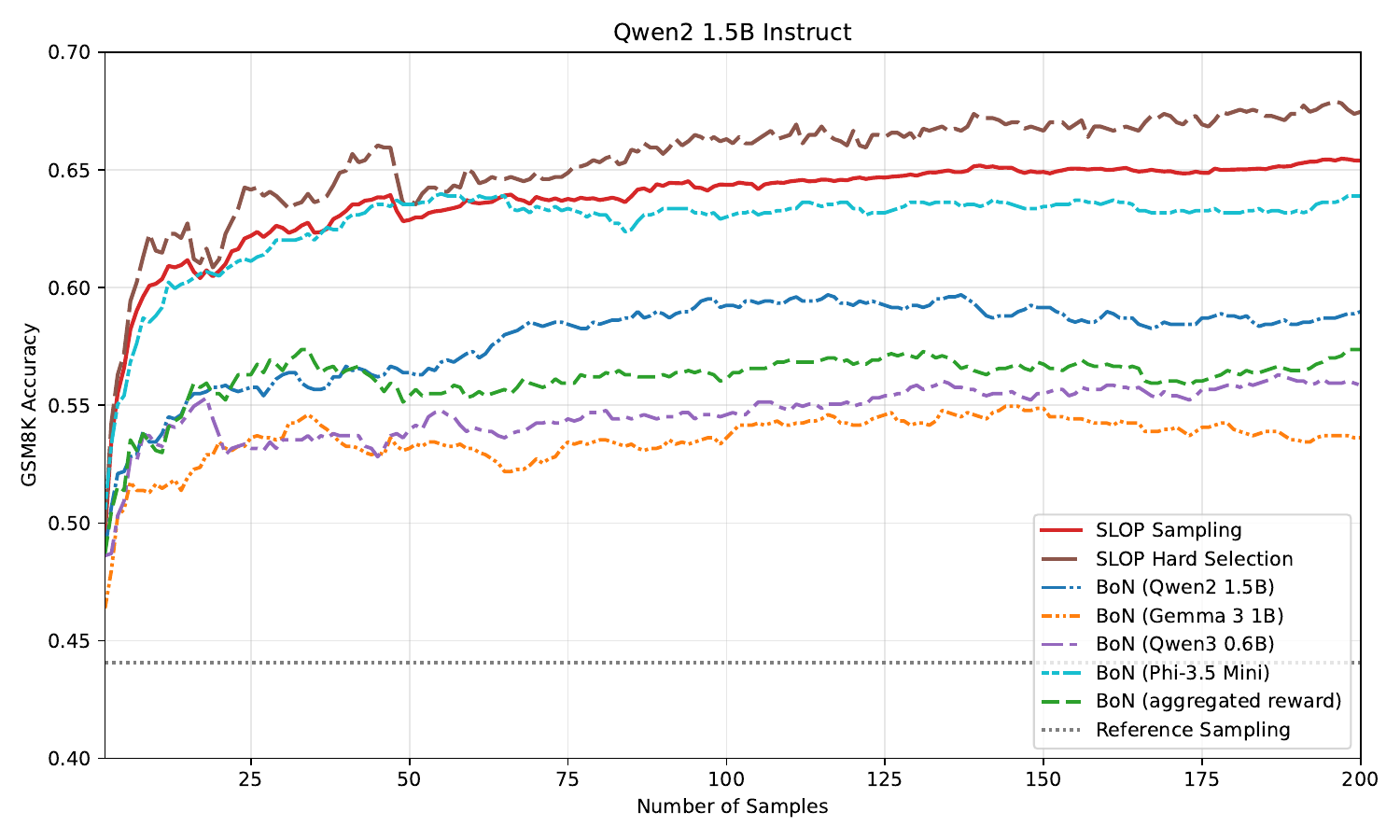}
\caption{SLOP (with 4 LLMs) evaluated on GSM8K, with Qwen2-1.5B as the reference model.}
\label{fig:gsm8k_qwen2}
\end{center}
\end{figure*}

\begin{figure*}[ht]
\begin{center}
\includegraphics[width=\textwidth]{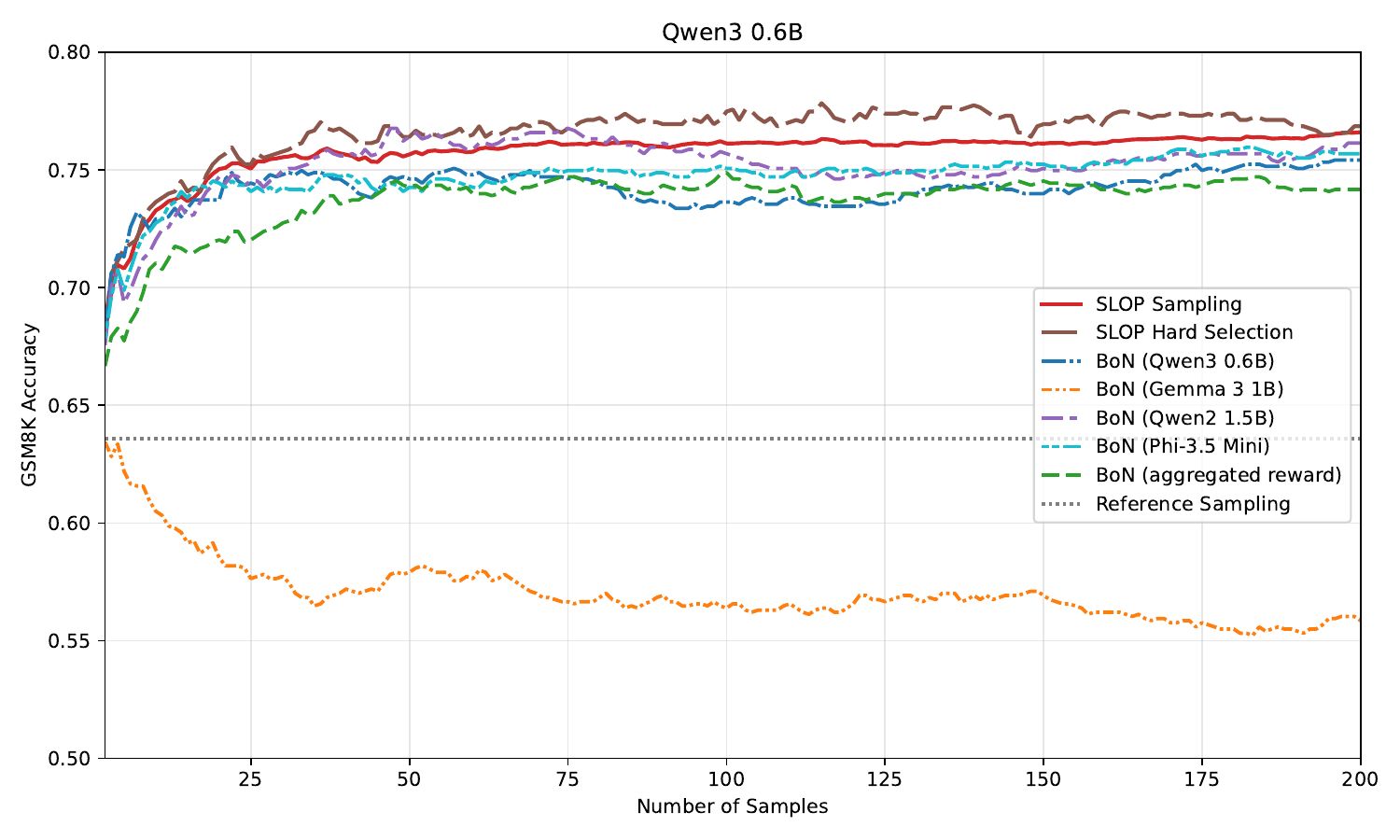}
\caption{SLOP (with 4 LLMs) evaluated on GSM8K, with Qwen3-0.6B as the reference model.}
\label{fig:gsm8k_qwen3}
\end{center}
\end{figure*}

\begin{figure*}[ht]
\begin{center}
\includegraphics[width=\textwidth]{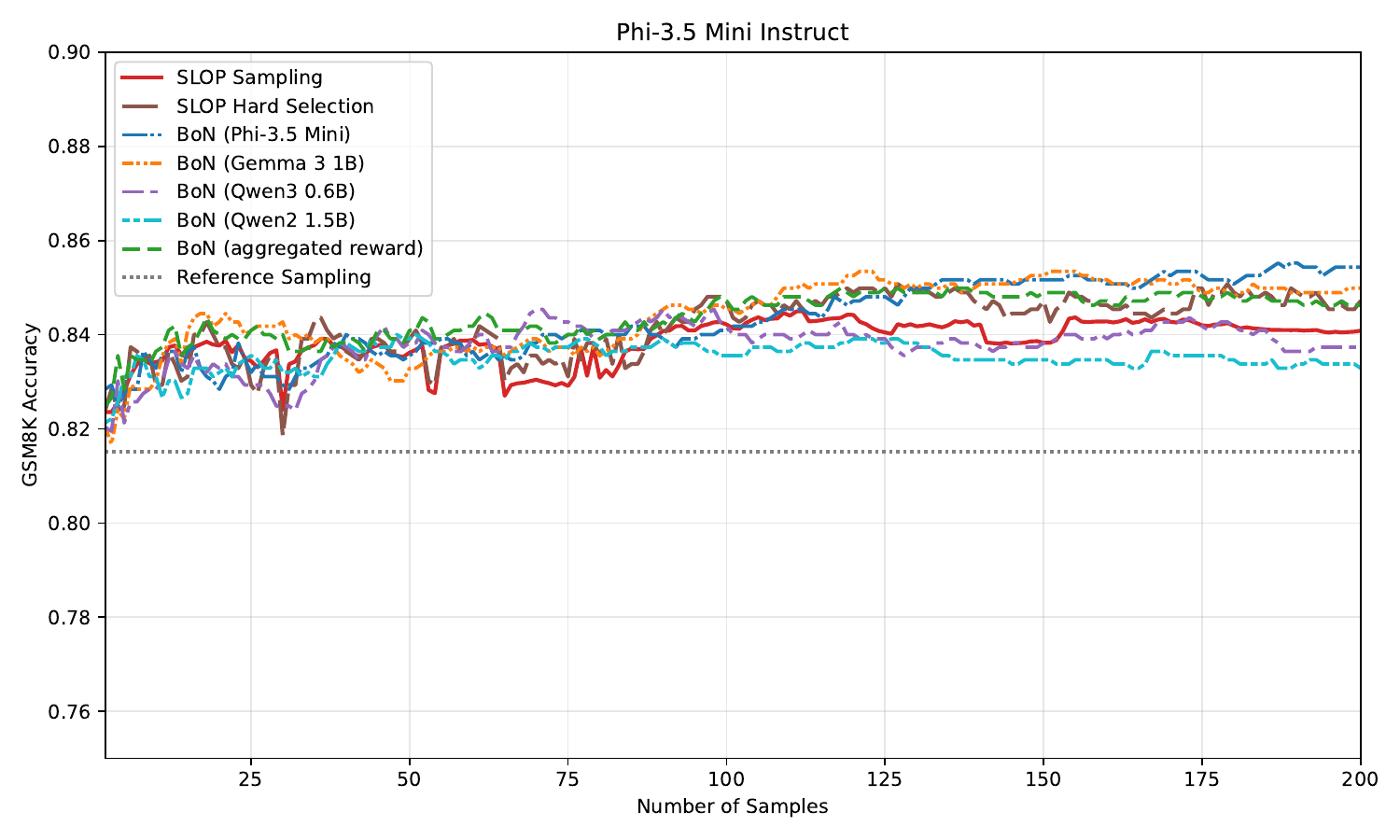}
\caption{SLOP (with 4 LLMs) evaluated on GSM8K, with Phi-3.5-mini as the reference model.}
\label{fig:gsm8k_phi3.5}
\end{center}
\end{figure*}

\begin{figure*}[ht]
\begin{center}
\includegraphics[width=\textwidth]{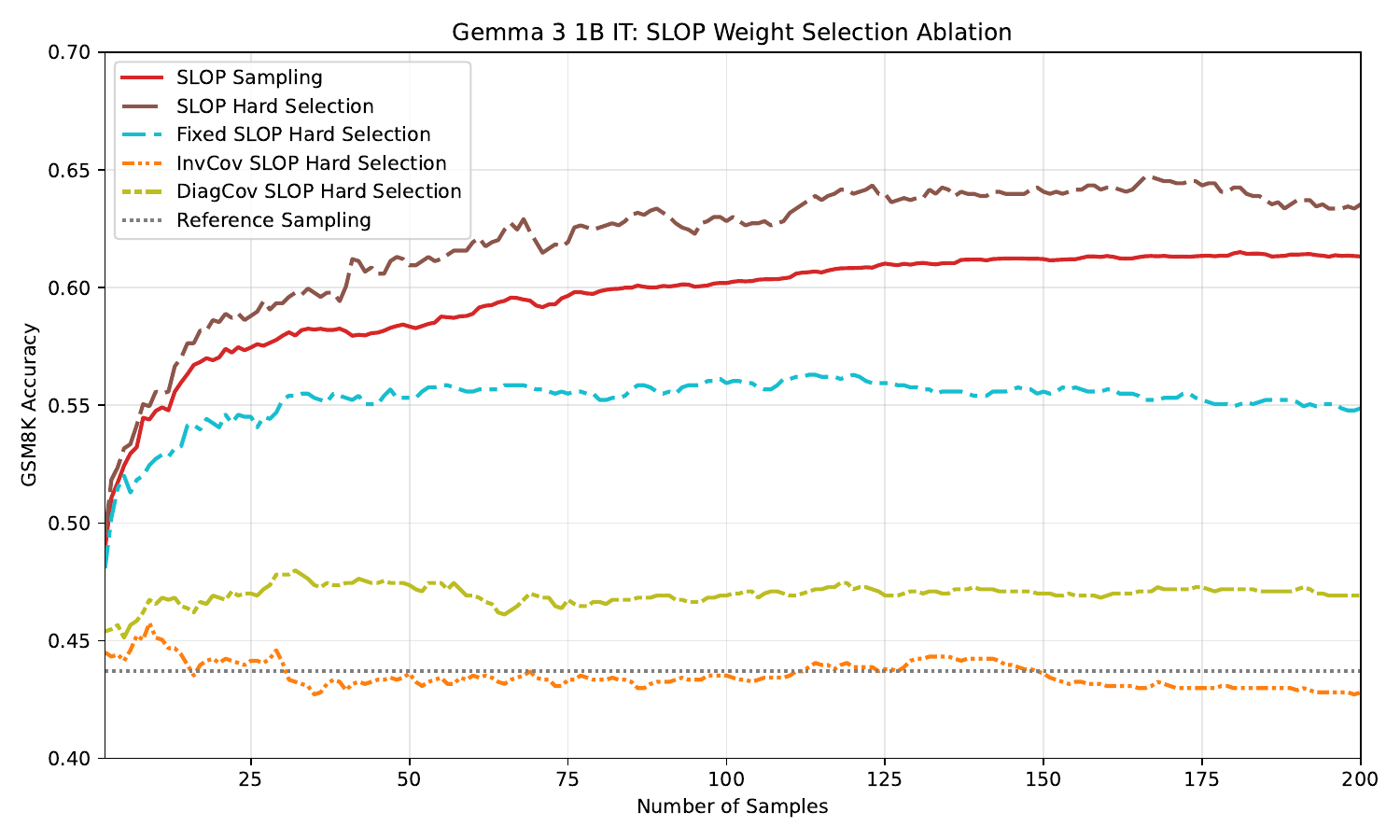}
\caption{SLOP ablations (with 4 LLMs) evaluated on GSM8K, with Gemma-3-1B as the reference model.}
\label{fig:gsm8k_gemma3_ablation}
\end{center}
\end{figure*}

\begin{figure*}[ht]
\begin{center}
\includegraphics[width=\textwidth]{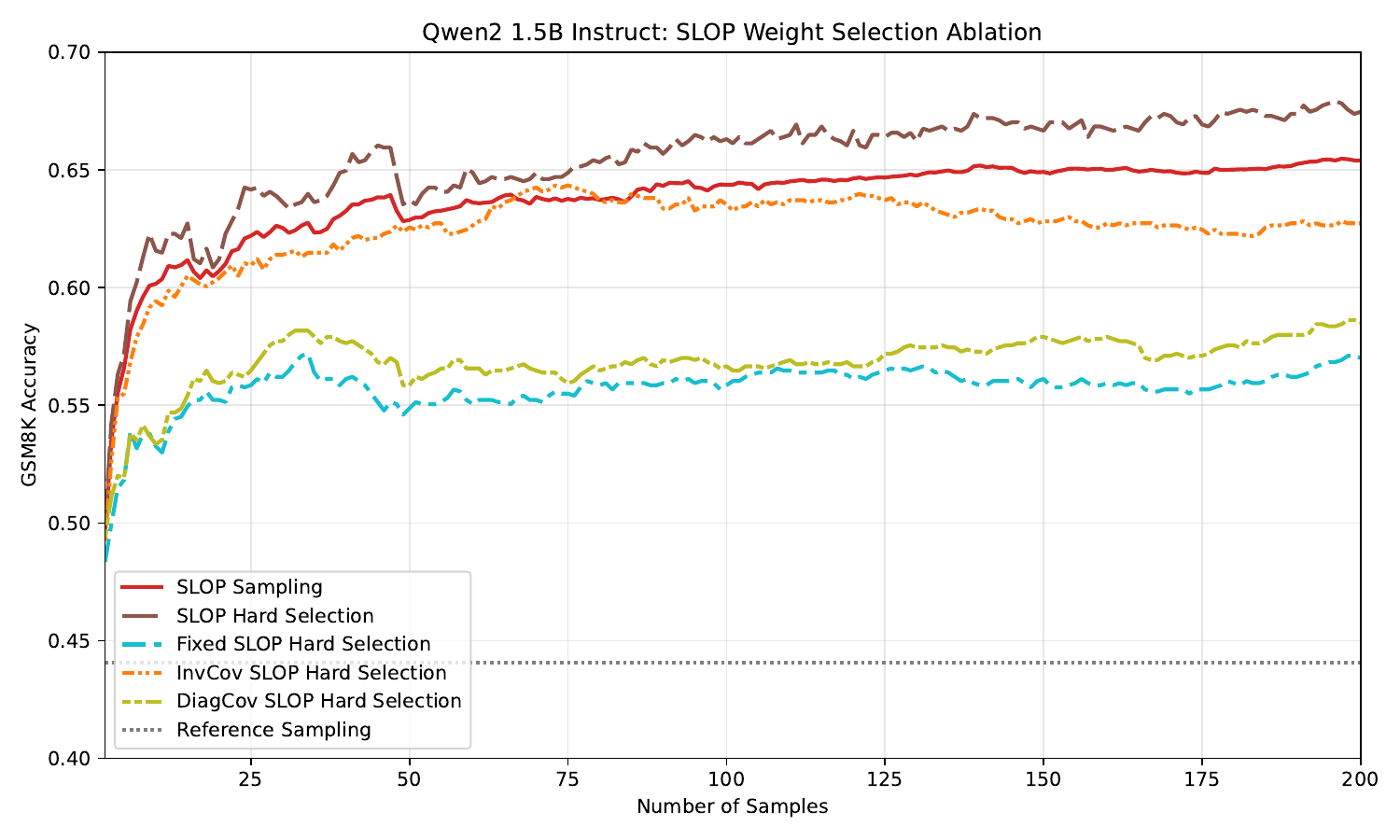}
\caption{SLOP ablations (with 4 LLMs) evaluated on GSM8K, with Qwen2-1.5B as the reference model.}
\label{fig:gsm8k_qwen2_ablation}
\end{center}
\end{figure*}

\begin{figure*}[ht]
\begin{center}
\includegraphics[width=\textwidth]{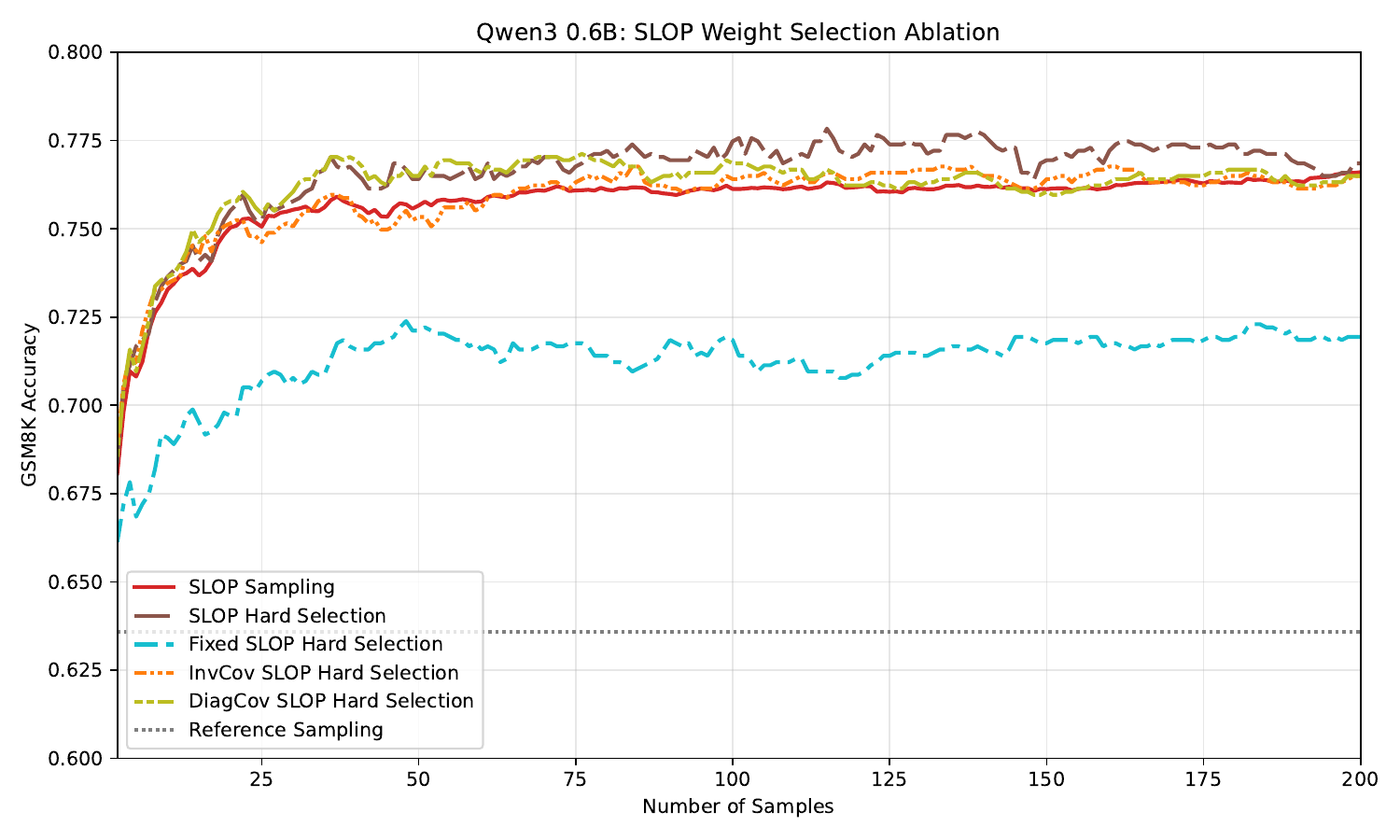}
\caption{SLOP ablations (with 4 LLMs) evaluated on GSM8K, with Qwen3-0.6B as the reference model.}
\label{fig:gsm8k_qwen3_ablation}
\end{center}
\end{figure*}

\begin{figure*}[ht]
\begin{center}
\includegraphics[width=\textwidth]{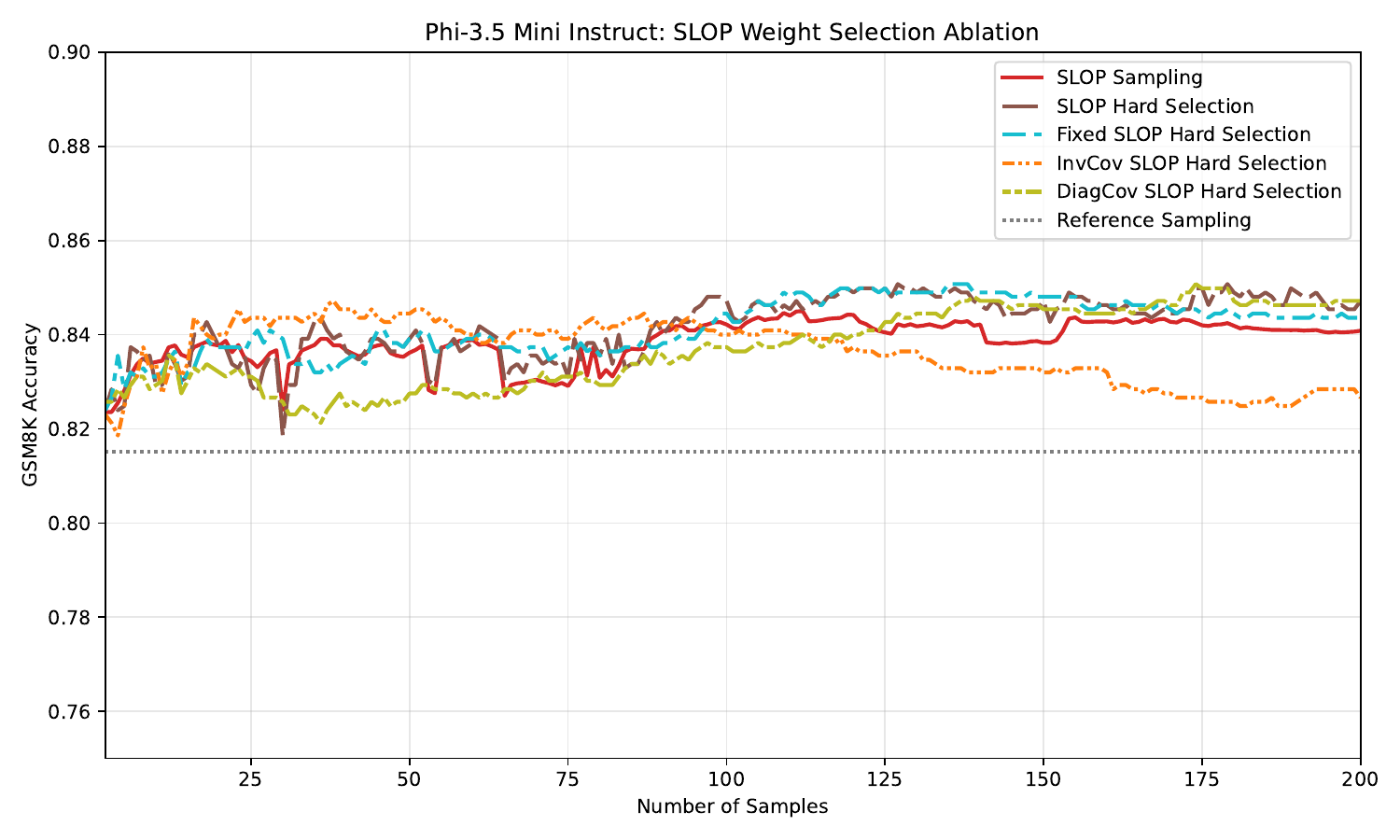}
\caption{SLOP ablations (with 4 LLMs) evaluated on GSM8K, with Phi-3.5-mini as the reference model.}
\label{fig:gsm8k_phi3.5_ablation}
\end{center}
\end{figure*}

\section{Math reasoning experiment details and additional results}
\label{app:more_gsm8k}

Responses are generated for up to $1024$ tokens, with the following system instruction, preceding each GSM8K question provided as the prompt: 

\verb|    You are a careful math tutor. Solve the user's grade-school math|
\verb|    problem, show your reasoning, and end with 'Final answer: <number>'.|

To check the correctness of a response, the answer is extracted via a regular-expression to find a number prefixed by ``Final answer:'', while falling back to extracting the last number in the response, if the prefix is not found, and then checking against the reference answer via string matching.

Weight optimization is performed for $T=500$ steps, with learning rate $\eta = 0.05$ and weight decay parameter $\lambda = 10^{-5}$, and by using Adam, with $\beta_1 = 0.9$ and $\beta_2 = 0.999$.

For convenience, Figures~\ref{fig:gsm8k_gemma3},~\ref{fig:gsm8k_qwen2},~\ref{fig:gsm8k_qwen3}, and~\ref{fig:gsm8k_phi3.5} reproduce the earlier Figure~\ref{fig:gsm8k_all}, but in larger format.

Figures~\ref{fig:gsm8k_gemma3_ablation},~\ref{fig:gsm8k_qwen2_ablation},~\ref{fig:gsm8k_qwen3_ablation}, and~\ref{fig:gsm8k_phi3.5_ablation} depict further SLOP experiment ablation results with weights chosen via the calibration score mean and covariance, or with fixed uniform weights, as described in Section~\ref{sec:gsm8k_covariance}.
For these ablations, we generally applied hard selection.

\section{Datasets and models}
\label{app:datasets_models}

For the visual question answering experiments in Section~\ref{sec:SQA_exp} and Appendix~\ref{app:more_SQA}, we used the following dataset and VLMs:
\begin{itemize}
\item
ScienceQA \citep{lu2022learn}: 
\begin{itemize}
\item License: Creative Commons Attribution Share Alike 4.0 (CC-BY-SA-4.0)
\item \url{https://huggingface.co/datasets/derek-thomas/ScienceQA}
\end{itemize}
\item
LLaVa-1.5-7B \citep{liu2023llava, liu2024improved}:
\begin{itemize}
\item License: Llama 2 Community License Agreement
\item \url{https://huggingface.co/llava-hf/llava-1.5-7b-hf}
\end{itemize}
\item
Gemma-3-4B \citep{team2025gemma}:
\begin{itemize}
\item License: Gemma Terms of Use
\item \url{https://huggingface.co/google/gemma-3-4b-it}
\end{itemize}
\item
Qwen3-VL-4B \citep{bai2025qwen3}:
\begin{itemize}
\item License: Apache License 2.0
\item \url{https://huggingface.co/Qwen/Qwen3-VL-4B-Instruct}
\end{itemize}
\end{itemize}

For the math reasoning experiments in Sections~\ref{sec:gsm8k_exp} and~\ref{sec:gsm8k_covariance}, and Appendix~\ref{app:more_gsm8k}, we used the following dataset and LLMs:
\begin{itemize}
\item
GSM8K \citep{cobbe2021training}:
\begin{itemize}
\item License: MIT License
\item \url{https://huggingface.co/datasets/openai/gsm8k}
\end{itemize}
\item
Gemma-3-1B \citep{team2025gemma}:
\begin{itemize}
\item License: Gemma Terms of Use
\item \url{https://huggingface.co/google/gemma-3-1b-it}
\end{itemize}
\item
Qwen2-1.5B \citep{yang2024qwen2}:
\begin{itemize}
\item License: Apache License 2.0
\item \url{https://huggingface.co/Qwen/Qwen2-1.5B-Instruct}
\end{itemize}
\item
Qwen3-0.6B \citep{yang2025qwen3}:
\begin{itemize}
\item License: Apache License 2.0
\item \url{https://huggingface.co/Qwen/Qwen3-0.6B}
\end{itemize}
\item
Phi-3.5-mini \citep{abdin2024phi3}:
\begin{itemize}
\item License: MIT License
\item \url{https://huggingface.co/microsoft/Phi-3.5-mini-instruct} 
\end{itemize}
\end{itemize}

\section{Computational resource usage}
\label{app:compute}

All experiments were run on a single desktop workstation, equipped with an Intel Core i7-13700K CPU, 128 GB of system RAM, and an Nvidia RTX 4090 GPU with 24 GB of VRAM.
The storage footprint for all models, datasets, and results generated by our experiments is approximately 48 GB.

In total, our experiments used approximately $150$ to $200$ hours of compute on this machine.
The vast majority of this time was used by the math reasoning experiments, discussed in Sections~\ref{sec:gsm8k_exp} and~\ref{sec:gsm8k_covariance}, and Appendix~\ref{app:more_gsm8k}, specifically response generation and scoring.
The visual question answering experiments of Section~\ref{sec:SQA_exp} and Appendix~\ref{app:more_SQA} consisted of no more than $3$ hours of the total compute, since our approach for these multiple-choice required only computing single next-token likelihoods.

\end{document}